\documentclass[12pt]{article}
\usepackage{fullpage}

\usepackage{amsmath,amsfonts,amssymb}
\usepackage{algorithm,algorithmic}
\usepackage{multirow}
\usepackage{cellspace}
\usepackage{setspace}
\usepackage{bm}
\usepackage{bbm}
\usepackage{subfigure}
\usepackage{graphicx}
\usepackage{tikz,pgfplots}
\usepackage{framed}

\usepackage[round]{natbib}

\usepackage{hyperref}
\hypersetup{colorlinks,
            linkcolor=blue,
            citecolor=blue,
            urlcolor=magenta,
            linktocpage,
            plainpages=false}

\usepackage[capitalise]{cleveref}

\usepackage{times}
\usepackage{graphicx} 
\usepackage{subfigure} 

\usepackage{natbib}

\usepackage{algorithm}
\usepackage{algorithmic}

\usepackage{hyperref}

\usepackage{geometry}                
\usepackage[parfill]{parskip}    
\usepackage{subfigure}
\usepackage{graphicx}
\usepackage{amssymb}
\usepackage{epstopdf}
\usepackage{algorithm}
\usepackage{algorithmic}
\usepackage{amsthm,amsmath,amssymb}
\usepackage{xcolor}

\newcommand{\ignore}[1]{}

\newcommand{\eq}{~=~}
\newcommand{\lee}{~\le~}
\newcommand{\gee}{~\ge~}
\newcommand{\lr}[1]{\!\left(#1\right)\!}

\newcommand{\lrset}[1]{\left\{#1\right\}}

\newtheorem{theorem}{Theorem}

\newtheorem{lemma}[theorem]{Lemma}
\newtheorem{corollary}[theorem]{Corollary}
\newtheorem{definition}[theorem]{Definition}

\newtheorem*{theorem*}{Theorem}
\newtheorem*{lemma*}{Lemma}
\newtheorem*{corollary*}{Corollary}
\newtheorem*{definition*}{Definition}
\newtheorem*{claim*}{Claim}
\newtheorem*{fact*}{Fact}

\DeclareMathOperator*{\argmin}{arg\,min}

\newcommand{\be}{\begin{align}}
\newcommand{\en}{\end{align}}

\newcommand{\ben}{\begin{align*}}
\newcommand{\enn}{\end{align*}}

\renewcommand{\t}[1]{\tilde{#1}}

\newcommand{\reals}{\mathbb{R}}

\newcommand{\D}{\mathcal{D}}

\renewcommand{\O}{O}

\newcommand{\tO}{\t{\O}}
\newcommand{\tOmega}{\t{\Omega}}
\newcommand{\E}{\mathbf{E}}

 \newcommand{\tens}{\otimes}
 \newcommand{\btens}{\bigotimes}
 \newcommand{\sign}{\text{sign}{}}

\newcommand{\A}{\mathcal{A}}
\newcommand{\F}{\mathcal{F}}
\newcommand{\G}{\mathcal{G}}
\newcommand{\hatg}{\hat{g}}
\newcommand{\tg}{\tilde{g}}
\newcommand{\tf}{\tilde{f}}

\newcommand{\alphat}{\alpha_{t}}
\newcommand{\alphatp}{\alpha_{t+1}}
\newcommand{\alphaz}{\alpha_{0}}
\newcommand{\Deltat}{\Delta_t}
\newcommand{\N}{\mathcal{N}}
\newcommand{\I}{\mathcal{I}}

\title{The Power of Normalization: Faster Evasion of Saddle Points}%

\author{%
Kfir Y. Levy\footnote{Department of Computer Science, ETH Z\"urich. 
Email :\texttt{yehuda.levy@inf.ethz.ch}.}
}

\date{October 2016}                                           

\begin{document}
\maketitle

\begin{abstract} 
A commonly used heuristic in non-convex optimization is   Normalized Gradient Descent (NGD) - a variant of gradient descent in which only the direction of the gradient is taken into account and its magnitude ignored. We analyze this heuristic and show that with carefully chosen parameters and noise injection, this method can provably evade saddle points. 
We establish the convergence of NGD to a local minimum, and demonstrate rates which
 improve upon the fastest known first order algorithm  due to \cite{ge}. 
 The effectiveness of our method  is demonstrated via an application to the problem of online tensor decomposition; a task  for which saddle point evasion is known to result in convergence to global minima. 
\end{abstract}

\section{Introduction}
Owing to its efficiency, simplicity and intuitive interpretation, Gradient Descent (GD)  and its numerous variants are often the  method of choice in 
large-scale optimization tasks, 
including neural network optimization \cite{bengio2009learning}, ranking \cite{ranking},  matrix completion \cite{jain2010guaranteed}, and reinforcement learning \cite{GradRL}. Normalized Gradient Descent (NGD) is a less popular alternation of GD, which enjoys the same  efficiency and simplicity.

Exploring the limitations of applying GD to non-convex tasks is an important an active line of research.
Several phenomena have been found to prevent GD from attaining  satisfactory results.  Among them are local-minima, and saddle points.
Local minima might entrap  GD, preventing it from reaching a possibly better solution. Gradients vanish near saddle points, which causes  GD to stall.
This saddle phenomenon was recently studied in the context of deep learning, where it was established both empirically and theoretically that saddles are prevalent in such tasks, and cause GD to decelerate \cite{saad1995exact,saxe2,BengioDauphin2014identifying,choromanska2015loss}. Several empirical studies suggest that arriving at  local minima is satisfactory for deep learning tasks: in \cite{BengioDauphin2014identifying}, it is asserted that all local minima are of approximately the same quality as the global one. Additionally, in \cite{choromanska2015loss} it is  argued that global minima may cause overfitting, while local minima are likely to yields sufficient generalization. Very recently, evading saddles was found to be crucial in other non-convex applications,
among them are complete dictionary recovery \cite{sun2015complete}, phase retrieval \cite{sun2016geometric}, and Matrix completion \cite{ge2016matrix}.

This paper studies the minimization of strict-saddle  functions $f:\reals^d \mapsto \reals$, which were previously introduced in the pioneering work of Ge et al., \cite{ge}.
Strict-saddle is a property that captures  multi-modality as well as the  saddle points phenomenon.
Motivated by \cite{choromanska2015loss,BengioDauphin2014identifying}, our goal is to approach at  some local minimum of $f$ while evading saddles. 
We show that  Saddle-NGD, a version of NGD which combines noise perturbations,  is more appropriate for this task than the noisy-GD algorithm proposed in \cite{ge}.
In particular, we show that in the offline  setting our method requires $\tO(\eta^{-3/2})$ iterations in order to reach a solution which is $\eta$-approximately (locally) optimal. This improves upon  noisy-GD which requires $\tO(\eta^{-2})$ iterations.
Moreover, we show that Saddle-NGD requires $\tO(d^3)$ iterations in order to arrive at a basin of some  local minimum, while noisy-GD  requires $\tO(d^4)$ iterations. In the stochastic optimization setting, we can extend our method to achieve the same sample complexity as noisy-GD.
Since a single iteration of NGD/GD costs the same, our results imply that  NGD   ensures an improved running rime over noisy-GD.

Our experimental study demonstrates the benefits of using Saddle-NGD for the setting of online tensor decomposition.
The tensor decomposition optimization problem contains both local minima and saddle points with the interesting property that any local minimum  is also a global minimum.

\subsection{Related Work}
Non-convex optimization problems are in general NP-hard and thus most literature on iterative optimization methods focuses on local guarantees. 
The most  natural guarantee to expect GD in this context is to approach a stationary point, namely a point where gradients vanish.
This kind of guarantees is provided in \cite{ allen2016variance, ghadimi2013stochastic}, which focus on the stochastic setting.
Approaching local minima while evading saddle points is  much more  intricate than ensuring stationary points.  
In  \cite{BengioDauphin2014identifying,sun2016geometric}, it is demonstrated how to avoid saddle 
points through a trust region approach which utilizes second order derivatives. 
A perturbed version of GD was  explored in \cite{ge}, showing both theoretically and empirically that it  efficiently evades saddles. Very recently, it was shown in \cite{AABHM2017} how to use second order methods in order to rapidly approach local minima. Nevertheless, this approach is only practical  in settings where Hessian-vector products can  be done efficiently.

NGD is well known to guarantee convergence in the offline settings of convex and quasi-convex optimization \cite{nesterov1984minimization,kiwiel2001convergence} .
 Recently, a stochastic version of NGD was explored in the context of  stochastic quasi-convex optimization, \cite{hazan2015beyond}.

\section{Setting and Assumptions}
We discuss the offline optimization setting, where we aim at minimizing a continuous smooth function $f:\reals^d\mapsto \reals$.
The optimization procedure lasts for $T$ rounds, at each round $t\in [T]$ we may query $x_t\in \reals^d$ and receive the gradient, $\nabla f(x_t)$, as a feedback. After the last round we aim at finding a point $\bar{x}_T$ such that $f(\bar{x}_T)-f(x^*)$ is small, where $x^*$ is some local minimum of $f$.

\paragraph{Strict-Saddle-Property:}
We focus on the optimization of strict-saddle functions, a family of multi-modal functions which encompasses objectives that acquire saddle points. Interestingly, \cite{ge} have found that tensor decomposition problems acquire this strict-saddle property.

\begin{definition}[Strict-saddle]
Given $\alpha,\gamma,\nu,r>0$ , a
 function $f$ is $(\alpha,\gamma,\nu,r)$-strict-saddle if for \textbf{any} $x\in \reals^d$ at least one of the following applies:
\begin{enumerate}
\item $\|\nabla f(x) \|\geq \nu$
\item $\lambda_{min}(\nabla^2 f(x))\leq -\gamma$
\item There exists a local minimum $x^*$ such that $\|x-x^* \|\leq r$, and the function $f(x)$ restricted to a $2r$ neighbourhood of  $x^*$ is $\alpha$-strongly-convex. 
\end{enumerate}
\end{definition}
here $\nabla^2 f(x)$ denotes the Hessian of $f$ at $x$, and $\lambda_{\min}(\nabla^2 f(x))$ is the minimal eigenvalue of
the Hessian.
According to the strict-saddle property, for any  $x\in \reals^d$  at least one of three scenarios applies:
either the gradient at $x$ is large, otherwise either we are close to a local minimum or
we are in the proximity of a saddle point.

We also assume that $f$ is \emph{$\beta$-smooth}, meaning:
\begin{align*}
\forall x,y \in \reals^d,\quad \|\nabla f(x) - \nabla f(y) \| \leq \beta \| x-y\|~.
\end{align*}
Lastly, we assume that $f$ has \emph{$\rho$-Lipschitz Hessians}:
\begin{align*}
\forall x,y \in \reals^d,\quad \|\nabla^2 f(x) - \nabla^2 f(y) \| \leq \rho \| x-y\|~.
\end{align*}
here $\| \cdot \|$ denotes the spectral norm when applied to matrices, and the $\ell_2$ norm when applied to vectors.

Finally, we recall the  definition of  strong-convexity:
\begin{definition}
A function $f: \reals^d\mapsto \reals$ is called $\alpha$-strongly-convex if for any $x,y\in\reals^d$ the following applies:
\begin{align*}
 f(y)-f(x) \geq \nabla f(x)^\top(y-x) +\frac{\alpha}{2}\| x-y\|^2~.
\end{align*}
\end{definition}

\section{Saddle-NGD algorithm and Main Results}
In Algorithm~\ref{algorithm:SNGD} we present Saddle-NGD, an algorithm that is adapted to handle strict-saddle functions. 
The main difference from GD is the  use of the direction of the gradients rather than the gradients themselves. Note that most updates are noiseless, yet once every $N_0$ rounds we add zero mean gaussian noise $\theta n_t, \text{where } n_t \sim \N(0,I_d)$, here $I_d$ is the $d$-dimensional identity matrix. 
The noisy updates ensure that once we arrive near a saddle, the direction with most negative eigenvalue will be sufficiently large. This in turn implies a fast evasion of the saddle.
\begin{algorithm}[t]
\caption{Saddle-Normalized Gradient Descent (Saddle-NGD) }
\label{algorithm:SNGD}
\begin{algorithmic}
\STATE \textbf{Input}:  $x_0\in \reals^d$,  $T_{\text{tot}}$, learning rate $\eta$ 
\STATE \textbf{Set}:  noise level $\theta=\tilde{\Theta}(\eta)$, noise period $N_0=\tilde{\Theta}(\eta^{-1/2})$  
\FOR{$t=0 \ldots T_{\text{tot}}$ }
\STATE \textbf{Let} $g_t := \nabla f(x_t)$, and $\hat{g}_t = \frac{g_t}{\|g_t\|}$ 
      \STATE \textbf{if} $(t\;  \text{mod}\; N_0) = 0$  \textbf{then} $n_t \sim\N(0,\mathcal{I}_d)$ \textbf{else} $n_t=0$  \textbf{end if}
   \STATE \textbf{Update}
      \begin{equation*}
           x_{t+1}= x_{t}-\eta \hat{g}_{t}+\theta n_t
        \end{equation*}
   \ENDFOR
\end{algorithmic}
\end{algorithm}

Following is the main theorem of this paper:
\begin{theorem}\label{theorem:MAinPaper}
Let $\xi \in (0,1)$,  $\eta \in (0,{\eta_{\max}})$. Also assume that $f:\reals^d\mapsto\reals$ is $(\alpha,\gamma,\nu,r)$-strict-saddle, $\beta$-smooth,  has $\rho$-Lipschitz Hessians, and also $|f(x)|\leq B;\; \forall x\in\reals^d$.
Then  w.p.$\geq 1-\xi$, Algorithm~\ref{algorithm:SNGD} reaches a point which is $\tO(\sqrt{\eta})$-close to some local minimum $x^*$ of $f$ within $T_{\text{tot}}=\tO(\eta^{-3/2})$ optimization steps. Moreover:
$$f(x_{T_{\text{tot}}})-f(x^*)\leq \tO(\eta)~. $$
\end{theorem}
The above statement holds for a sufficiently small learning rate $\eta\leq \eta_{\max}$. 
Our analysis shows that  $\eta_{\max} = \tO(1/d^2)$. We improve upon noisy-GD proposed in \cite{ge}, which requires $\tO(\eta^{-2})$ iterations to reach a point $\tO(\eta)$ approximately (locally) optimal, and acquired the same dependence of  $\eta_{\max} = \tO(1/d^2)$.

\textbf{Notation:} Our $\tO(\cdot),\tOmega(\cdot),\tilde{\Theta}(\cdot)$ notation conveys polynomial dependence on $\eta$ and $d$. It hides polynomial factors in $\alpha,\gamma,\beta,\nu,r,\rho,B, \log(1/\eta),\log(d),\log(1/\xi)$.

In fact, since strict-saddle functions are strongly-convex in a $2r$ radius around local minima, our ultimate goal should be arriving $r$-close to such minima. Then we could use  the well known convex optimization machinery to rapidly converge.
The following corollary formalizes the number of rounds required to reach  $r$-close to some local minimum:
\begin{corollary}\label{cor:MAinPaper}
Let $\xi \in (0,1)$. Assume $f$ as in Theorem~\ref{theorem:MAinPaper}, then setting 
$\eta = \min\{\eta_{\max},\tO(r^2)\}$ Algorithm~\ref{algorithm:SNGD} reaches $r$-close to some local minimum of $f$ within 
$T_{\text{tot}}=\tO(\eta_{\max}^{-3/2}) = \tO(1/d^3)$ steps, w.p.$\geq 1-\xi$.
\end{corollary}
Note that the above improves over  noisy-GD which requires $\tO(d^4)$ steps in order to reach a point  $r$-close to some local minimum of $f$ .
\subsection{Stochastic Version of Saddle-NGD}
In the \emph{Stochastic Optimization} setting,  an exact gradient for the objective function is unavailable, and we may only access noisy (unbiased) gradient estimates. This  setting captures some of the most important 
tasks in machine learning, and is therefore the subject of an extensive study.

Our Saddle-NGD algorithm for offline optimization of strict-saddle functions could be extended to the stochastic setting. This could be done by using minibatches in order to calculate the gradient, i.e. instead of $g_t = \nabla f(x_t)$ appearing in 
Algorithm~\ref{algorithm:SNGD}, we use:
$$g_t : = \frac{1}{b}\sum_{i=1}^b \G(x_t,\zeta_i) ~,$$
here $\{\G(x_t,\zeta_i)\}_i$ are independent and unbiased estimates of $\nabla f(x_t)$. Except for this alternation, the stochastic version of Saddle-NGD is similar to the one presented in Algorithm~\ref{algorithm:SNGD}.
We have found that in order to ensure convergence in the stochastic case, a minibatch $b = \text{poly}(1/\eta)$ is required.
This dependence implies that in the stochastic setting Saddle-NGD obtains no better guarantees
than noisy-GD. We therefore omit the proof for the stochastic setting.

Nevertheless, we have found that in practice, the stochastic version of Saddle-NGD demonstrates a superior performance  compared to noisy-GD. 
Even when we employ a moderate minibatch size.
This is illustrated in Section~\ref{sec:Experiments} where Saddle-NGD is applied to the task of online tensor decomposition.

\section{ Analysis Overview }
Our analysis of Algorithm~\ref{algorithm:SNGD} divides according to the three scenarios defined by the strict-saddle property. Here we present the main statements regarding the guarantees of the algorithm for each scenario, and  provide a proof sketch of Theorem~\ref{theorem:MAinPaper}.
For ease of notation we assume that we reach at each scenario at $t=0$. 

In case that  the  gradient\footnote{Note the use of the notation $g_t = \nabla f(x_t)$ 
here and in the rest of the paper.} 
is sufficiently large, the following ensures us to improve by value within one step:
\begin{lemma}\label{lem:LargeGrads}
Suppose that $\|g_0\|\geq \beta \sqrt{\eta}$, and we use the Saddle-NGD Algorithm, then: 
\begin{align*}
\E[f(x_1)-f(x_0)\mid x_0 ]\lee -\beta \eta^{3/2}/2~.
\end{align*}
\end{lemma}
The next lemma ensures that once we arrive close enough to a local minimum we will remain in its proximity. 
\begin{lemma}\label{lem:LocalOptFull}
Suppose that $x_0$ is $r$-close to a local minimum, i.e.,  $ \|x_0 -x^*\| \leq  r $,  and $\|g_0\|\leq \beta \sqrt{\eta} \leq \nu$. Then the following holds for any $0\leq t \leq \tO(\eta^{-3/2})$, w.p.$\geq 1-\xi$:
\begin{align*}
\|x_{t}-x^*\|^2 \lee \max\{ \|x_0-x^*\|^2,\frac{\beta^2}{\alpha^2}\eta\} + \tO(\eta)~.
\end{align*}
\end{lemma}
The next statement describes the improvement attained by Saddle-NGD  near  saddle points.
\begin{lemma}\label{lem:SaddleGeneralMain}
Let $\xi\in[0,1]$. 
Suppose that $\|g_0\|\leq \beta \sqrt{\eta}\leq \nu$ and we are near a saddle point, meaning that $\lambda_{min}(\nabla^2 f(x_0))\leq -\gamma$.
Then w.p.$\geq 1-\xi$, within  $t\leq \tO( \eta^{-1/2})$ steps, we will have :
$$f(x_t) \lee  f(x_0) -  \tOmega(\eta)~. $$
\end{lemma}
Theorem~\ref{theorem:MAinPaper} is based on the above three lemmas. The full proof of the Theorem appears in Appendix~\ref{Proof_theorem:MAinPaper}.  Next we provide a short sketch of the proof.
\vspace{-7pt}
\begin{proof}[Proof Sketch of Theorem~\ref{theorem:MAinPaper}]
Loosely speaking, Lemmas~\ref{lem:LargeGrads},~\ref{lem:SaddleGeneralMain} imply that as long as we have not reached close to a local minimum, then our per-round improvement is   $\tOmega(\eta^{3/2})$  (on average). Since $f$ is bounded, this implies that within $\tO(\eta^{-3/2})$ rounds we reach at some local minimum, $x^*$. Lemma~\ref{lem:LocalOptFull} ensures that we will remain in the $\tO(\sqrt{\eta})$ proximity of this local minimum.  Finally, since $f$ is $\beta$-smooth, this $\tO(\sqrt{\eta})$ proximity  implies that we reach a point which is $\tO(\eta)$ close by value to $f(x^*)$.
\end{proof}

\vspace{-18pt}
\section{Analysis}
\vspace{-7pt}
Here we prove the main statements regrading the three scenarios defined by the strict-saddle property. Section~\ref{sec:largeGrads} analyses the scenario of large gradients (Lemma~\ref{lem:LargeGrads}), Section~\ref{sec:LocalMin} analyses the local-minimum scenario (Lemma~\ref{lem:LocalOptFull}), and Section~\ref{sec:SaddleMetaSection} analyses the case of saddle points (Lemma~\ref{lem:SaddleGeneralMain}). For brevity we do not always provide full proofs, which are deferred to appendix. 
\vspace{-7pt}
\subsection{Large Gradients}\label{sec:largeGrads}
\vspace{-7pt}
\begin{proof}[Proof of Lemma~\ref{lem:LargeGrads}]
We will prove assuming a noisy update, i.e.  $x_{1}=x_0-\eta \hatg_0+\theta n_0$ and $n_0\sim \N(0,I_d)$ (the noiseless update case is similar). By the update rule:
\begin{align*}
\E[f(x_{1})-f(x_0)\mid x_0] 
&\lee \E [ g_0^\top(x_{1}- x_{0})+\frac{\beta}{2}\| x_{1}-x_{0}\|^2 \mid x_0 ] \\
& \eq \E[ g_0^\top (-\eta \hatg_0+\theta n_0) + \frac{\beta}{2} \|- \eta \hatg_0+\theta n_0 \|^2 \mid x_0 ] \\
&\eq -\eta \|g_0 \|+\frac{\beta}{2}(\eta^2+d\theta^2)\\
&\lee -\beta \eta^{3/2}  +\frac{\beta}{2}(\eta^2+d\theta^2)\\   
&\lee -\beta \eta^{3/2}/2 \\ 
\end{align*}
here we used $\E[n_0]=0$, $\E \| n_0\|^2=d$, the smoothness of $f$, $\|\hatg_0\|=1$, and the last inequality uses $d\theta^2=d\tilde{\Theta}(\eta^2)< \eta^{3/2}/4$, which holds if we choose $\eta= \tO(1/d^2)$.
In order for the other scenarios to hold when the gradient is small we  require $\beta \sqrt{\eta} \leq \nu$.
\end{proof}

\subsection{Local Minimum} \label{sec:LocalMin}
For brevity we will not prove Lemma~\ref{lem:LocalOptFull}, but rather state and prove a simpler lemma assuming all updates are noiseless; the proof of  Lemma~\ref{lem:LocalOptFull} regarding the general case appears in Appendix~\ref{sec:Omitted Proof Local Minimum}.
\begin{lemma}\label{lem:LocalOpt}
Suppose that $x_0$ is close to a local minimum $x^*$, i.e.,  $ \|x_0 -x^*\| \leq  r $,  and $\|g_0\|\leq \beta \sqrt{\eta} \leq \nu$. Then the following holds for any $t\geq 0$:
\begin{align*}
\|x_{t}-x^*\|^2 \lee \max\{ \|x_0-x^*\|^2,\frac{2\beta^2}{\alpha^2}\eta^2\}~.
\end{align*}
\end{lemma} 
\begin{proof}
Due to the local strong-strong convexity of $f$ around $x^*$,
we know that $ \|x_0 -x^*\| \leq \frac{1}{\alpha}\|g_0\| \leq \frac{\beta}{\alpha} \sqrt{\eta}$. In order to be consistent with the definition of strict-saddle property we choose $\eta$ such that
$ \frac{\beta}{\alpha} \sqrt{\eta} \leq r$.

The proof requires the following lemma regarding strongly-convex functions:
\begin{lemma}\label{lem:StrCvxity}
Let $F:\reals^d\mapsto \reals$ be an $\alpha$-strongly convex function,  let $x^*=\argmin_{x\in \reals^d} F(x)$ then the following holds for any $x\in \reals^d$:
\begin{align*}
 \nabla F(x)^\top(x-x^*)\geq  \alpha \| x-x^*\|^2~.
\end{align*}
\end{lemma}

We are now ready to prove  Lemma~\ref{lem:LocalOpt} by induction, assuming all updates are noiseless, i.e.
$x_{t+1} = x_{t}-\eta \hatg_{t}$.
Note that the case $t=0$ naturally holds, next we discuss the case where $t\geq 1$.
First assume that  $\|x_{t}-x^*\|\geq \frac{\beta}{\alpha}\eta$, the noiseless Saddle-NGD update rule implies:
\begin{align*}
\|x_{t+1}-x^*\|^2 
&\eq \|x_{t} -x^* \|^2 -2\eta \hat{g}_t^\top (x_{t}-x^*) + \eta^2 \\
& \eq  \|x_{t} -x^* \|^2 -2\eta \frac{1}{\|g_t\|} {g}_t^\top (x_{t}-x^*) + \eta^2\\
&\lee  \|x_{t} -x^* \|^2 -2\eta \frac{1}{\|g_t\|} \alpha \|x_{t}-x^*\|^2 + \eta^2 \\
&\lee \|x_{t} -x^* \|^2 -2\eta\frac{\alpha}{\beta}  \|x_{t}-x^*\| + \eta^2 \\
&\lee \|x_{t} -x^* \|^2   \\
&\lee \max\{ \|x_0-x^*\|^2,\frac{2\beta^2}{\alpha^2}\eta^2\}~,
\end{align*}
here the first inequality  uses Lemma~\ref{lem:StrCvxity},  the second inequality uses $\|g_t\|\leq \beta \|x_t-x^* \|$ which follows from smoothness, and the third inequality uses $\|x_{t}-x^*\|\geq \frac{\beta}{\alpha}\eta$.

For the case where $\|x_{t}-x^*\|\leq \frac{\beta}{\alpha}\eta$ similarly to the above  we can conclude that:
\begin{align*}
\|&x_{t}-x^*\|^2 \\
&\eq \|x_{t} -x^* \|^2 -2\eta \hat{g}_t^\top (x_{t}-x^*) + \eta^2 \\
&\lee \|x_{t} -x^* \|^2  + \eta^2\\
&\lee \frac{2\beta^2}{\alpha^2}\eta^2\\
&\lee \max\{ \|x_0-x^*\|^2,\frac{2\beta^2}{\alpha^2}\eta^2\} ~,
\end{align*}
we use $\hat{g}_t^\top (x_{t}-x^*)\geq 0$, which follows by the local strong-convexity around $x^*$.
\end{proof}
\subsection{Saddle Points}~\label{sec:SaddleMetaSection}
\vspace{-35pt}
\subsubsection*{Intuition and Proof Sketch}
We first provide some intuition regarding the benefits of using NGD rather than GD in escaping saddles. Later we present a short proof sketch of Lemma 
~\ref{lem:SaddleGeneralMain}.
\paragraph{Intuition: NGD vs. GD for a Pure Saddle}
Lemma~\ref{lem:SaddleGeneralMain} states the decrease in function values attained by the Saddle-NGD  near  saddle points.
Intuitively, Saddle-NGD implicitly performs an approximate  power method over the Hessian matrix $H_0: = \nabla^2 f(x_0)$. Since the gradients near saddle points tend to be  small  the use of normalized gradients rather than  the gradients themselves   yields a faster improvement.\\
Consider the minimization of  a pure saddle function: $F(x_1,x_2) = x_1^2 - x_2^2$.
As can be seen in Figure~\ref{fig:GD_pure}, the gradients are almost vanishing around the saddle point $(0,0)$.  Conversely, the normalized gradients (Figure~\ref{fig:NGD_pure}) are of  a constant magnitude.
This intuitively suggets that using NGD instead of (noisy) GD yields a faster escape of the saddle. Figure~\ref{fig:NGDvsGD_pure}  compares between NGD and (noisy) GD for the pure saddle function; both methods are initialized in the proximity of $(0,0)$, and employ a learning rate of $\eta = 0.01$. As  expected, NGD attains a much faster initial improvement. 
Later, when the gradients are sufficiently large, GD prevails. Since our goal is the  optimization of  a general family of functions where saddles behave like pure saddles only locally, we are mostly concerned about the initial  local improvement in the case of a pure saddles. This renders NGD  more appropriate than GD for our goal.
\begin{figure*}[t]
\centering
\subfigure[]{ \label{fig:GD_pure}
\includegraphics[trim = 30mm 75mm 27mm 65mm, clip,
width=0.3\textwidth ]{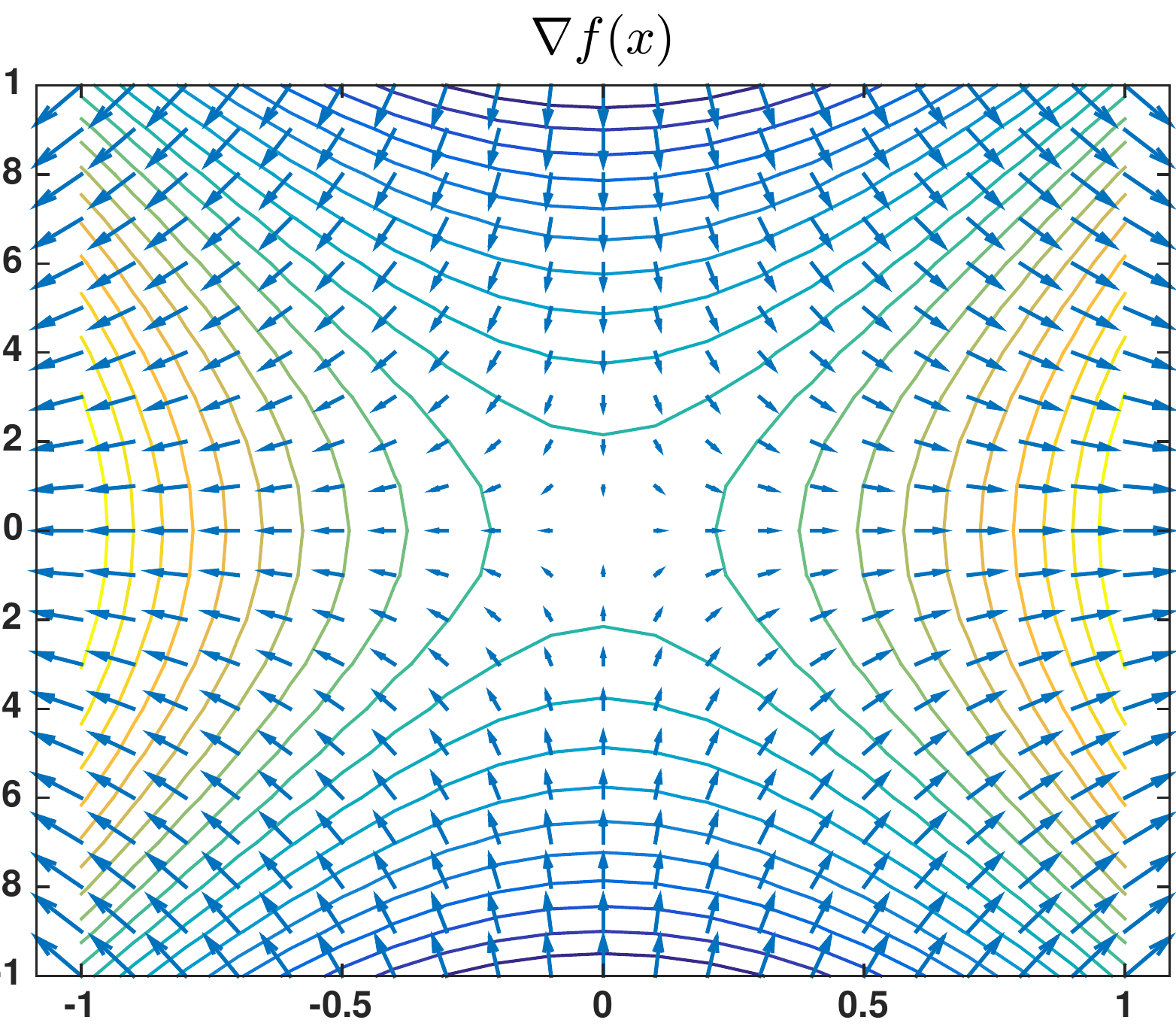}}
\subfigure[]{\label{fig:NGD_pure}
 \includegraphics[trim = 30mm 75mm 27mm 65mm, clip,
width=0.3\textwidth ]{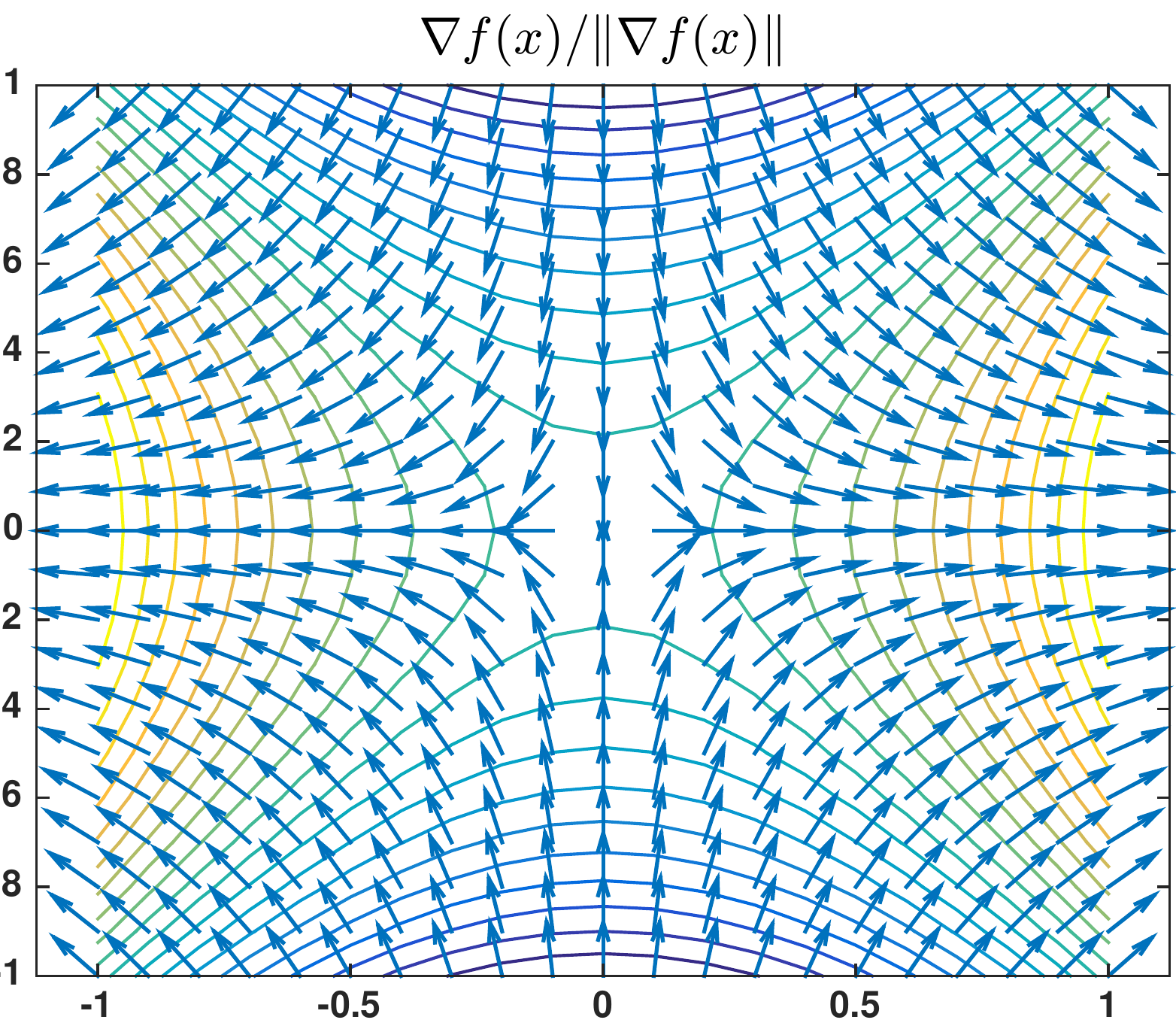}}
 \subfigure[]{ \label{fig:NGDvsGD_pure}
\includegraphics[trim = 10mm 65mm 21mm 73mm, clip, width=0.3\textwidth ]{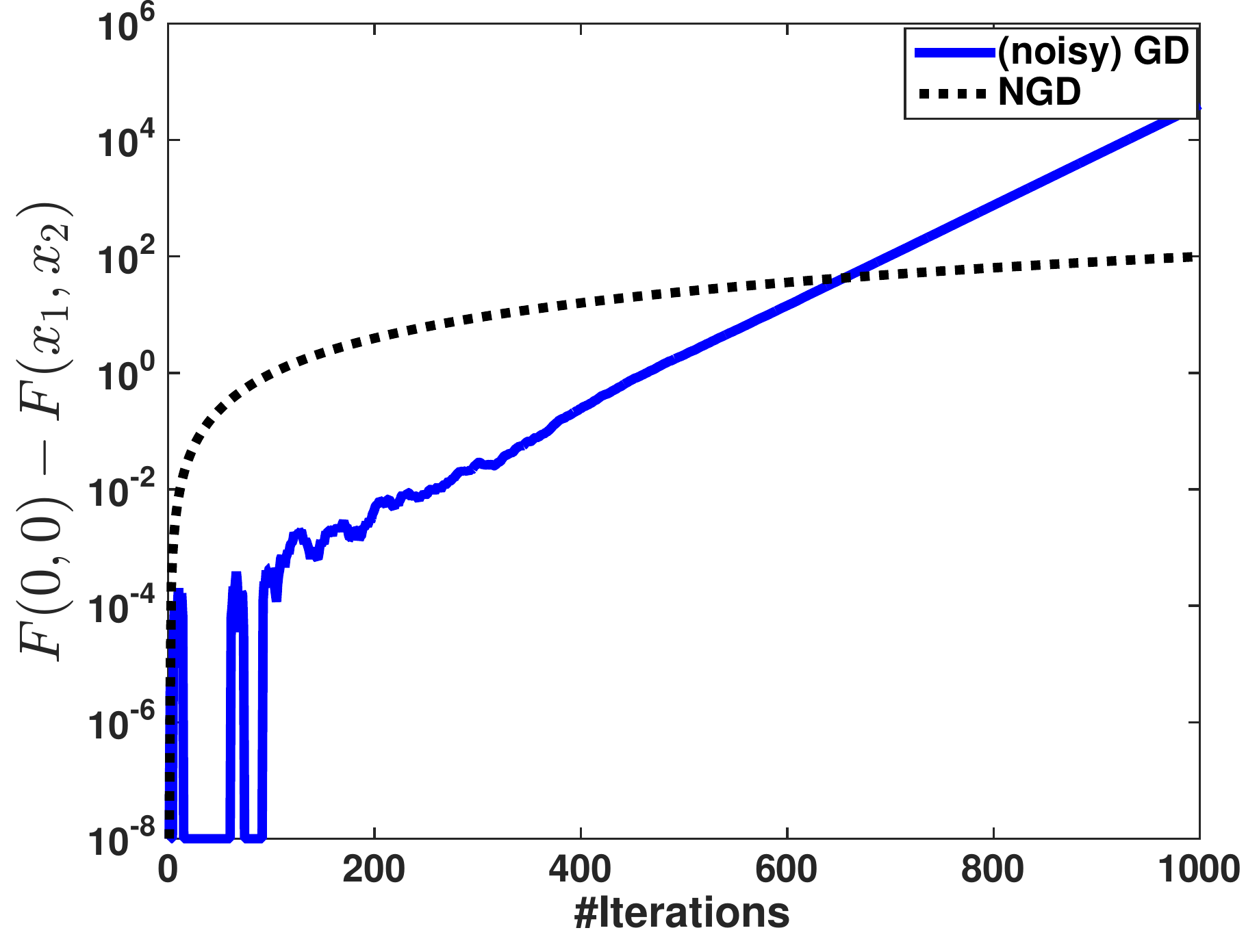}}
\caption{GD vs. NGD around a pure saddle. Left: gradients. Middle: normalized gradients. On the Right we compare GD against NGD, we present  $F(0,0)-F(x_1,x_2)$ vs. $\#$iterations. } 
\end{figure*}

\begin{proof}[Proof sketch of Lemma~~\ref{lem:SaddleGeneralMain}]  
Let $x_0$ be a point such that $\| \nabla f(x_0)\| \leq \beta\sqrt{\eta}$, and also let 
$\lambda_{\min}(\nabla^2 f(x_0))\leq -\gamma$.
Letting  $H_0 = \nabla^2 f(x_0)$, it can be shown that the 
 Saddle-NGD update rule implies:
\begin{align}\label{eq:PowMethodExplain}
\nabla f(x_t) 
\eq
\lr{I- \frac{\eta H_0}{\|\nabla f(x_{t-1}) \|}} \nabla f(x_{t-1}) +\mu_t- \theta H_0 n_t
~,
\end{align}
where $\mu_t \leq \tO(\eta^{3/2})$ for any $t \leq \tO(\eta^{-1/2})$.
Note that the above suggests that Saddle-NGD implicitly performs an approximate power method over the gradients of $f$, with respect to the matrix
$A_t:= \lr{I- \frac{\eta H_0}{\|\nabla f(x_{t-1}) \|}}$.\\
There are two differences compared to the traditional power method: first, the multiplying matrix changes from one round to another depending on 
$\|\nabla f(x_{t-1})\|$, which is a consequence of the normalization. Second, there is an additional additive term $\mu_t$, which is related to the deviation of the objective (and its gradients) from its second order taylor approximation around $x_0$.

When  the gradients are small, the dependence on $\|\nabla f(x_{t-1})\|$   amplifies the increase (resp. decrease) in the magnitude of the components related the negative (resp. positive) eigenvalues of $H_0$. Concretely, it can be shown that whenever $\|\nabla f(x_{t-1})\|\leq  \O(\sqrt{\eta})$, then 
the gradient component which is related to the eigenvector with the most negative eigenvalue of $H_0$, blows up by a factor of $1+\O(\sqrt{\eta})$ at every round
(ignoring $\mu_t$ and the noise term for now). This blow up factor means  that within $\tO(\eta^{-1/2})$ rounds this component increases beyond 
some threshold value ($\geq 2\beta \sqrt{\eta}$).

Intuitively, the increase in the magnitude of components related to the \emph{negative} eigenvalues decreases the objective's value.
 The proof goes on by showing that the
components of gradients/query-points related to the \emph{positive} eigenvectors of $H_0$, do not increase by much. This in turn allows to show that 
the objective's value decreases by $\tOmega(\eta)$ within $\tO(\eta^{-1/2})$ iterations.

The noise injections  utilized by   Saddle-NGD ensure that the additive term, $\mu_t$, does not interfere with the increase of the components 
related to the negative eigenvectors. Since $\mu_t$ is bounded, a careful choice of the  noise magnitude ensures that there is a sufficiently  large initial component in these directions.
\end{proof}

\subsubsection*{Analysis}

Before proving Lemma~\ref{lem:SaddleGeneralMain} we  introduce some notation and establish several lemmas regarding the dynamics of the Saddle-NGD update rule near a saddle point.

\paragraph{Quadratic approximation:}
Let $x_0$ be a point such that $\nabla f(x_0)\leq \beta\sqrt{\eta}$, and also 
$\lambda_{\min}(\nabla^2 f(x_0))\leq -\gamma$   (w.l.o.g. we assume $\lambda_{min}(\nabla^2 f(x_0))= -\gamma$).
Denote by $\tf(x)$ the quadratic approximation of $f$ around $x_0$, i.e.:
\begin{align}\label{eq:QuadApprox1}
\tf(x) \eq f(x_0)+ g_0^\top(x-x_0) + \frac{1}{2}(x-x_0)^\top H_0 (x-x_0)~.
\end{align}
here $g_0 = \nabla f(x_0)$,  $H_0 = \nabla^2 f(x_0)$.
For simplicity we assume that $H_0$ is full rank
\footnote{The analysis for the case $\text{rank}(H_0)<d$ is similar.
Moreover, we can always add $\tf$ an infinitesimally small random perturbation 
$x^\top H_\Delta x $ such that $rank(H_0 + H_\Delta) =d$, w.p.$1$, and  the magnitude  of 
perturbation is arbitrarily small for any $x$ relevant for the analysis}. 
This implies that  there exist
 $\bar{x}_0, C$, such that: $\tf(x) = C + \frac{1}{2}(x-\bar{x}_0)^\top H_0 (x-\bar{x}_0)$.
 Without loss of generality we assume  $\bar{x}_0=0, C=0$. Thus the quadratic approximation is of the following form:
\begin{align}\label{eq:QuadApprox2}
\tf(x) \eq  \frac{1}{2}x^\top H_0 x~.
\end{align}
Along this section, we will interchangeably use Equations~\eqref{eq:QuadApprox1} and \eqref{eq:QuadApprox2} for $\tf$.

Let $\{ e_i\}_{i=1}^d$ be an orthonormal eigen-basis of $H_0$ with eigenvalues
 $\{\lambda_i\}_{i=1}^d$:
\begin{align*}
H_0 \eq \sum_i \lambda_i e_i e_i^\top~.
\end{align*}
Also assume without loss of generality that $e_1$ is the direction with the most negative eigenvector, i.e. $\lambda_1 =  -\gamma$.
We  represent each point  in the eigen-basis of $H_0$, i.e., 
$$x_t \eq \sum_i \alphat^{(i)} e_i~,$$
and therefore 
$$\nabla \tf(x_t) \eq H_0 x_t \eq \sum_i \lambda_i \alphat^{(i)} e_i~.$$

Denoting $g_t^{(i)} = e_i^\top g_t,\;  n_t^{(i)} = e_i^\top n_t$; the NGD update rule:  $x_{t+1} = x_{t} - \eta \frac{g_t}{\|g_t \|}+\theta n_t$, translates coordinate-wise as follows:
\begin{align*}
\alphatp^{(i)} 
\eq
 \alphat^{(i)} -\eta \frac{g_t^{(i)}}{\| g_t\|}+\theta n_t^{(i)}~.
\end{align*}

\paragraph{Part 0:} Here we bound the difference between the gradients of the original function $f$ and the gradients of the quadratic approximation $\tf$, and show that these are bounded by the square norm of the distance between $x$ and $x_0$. This bound will be useful during the proofs since we decompose the the $x_t$'s and their gradients according to the Hessian at $x_0$.

For any $x\in \reals^d$, the gradient, $\nabla f(x)$, can be expressed as follows:
$$\nabla f(x) \eq \nabla f(x_0) +\int_{s=0}^1 H(x_0 +s(x-x_0)) ds(x-x_0) ~.$$
The above enables to relate the gradient of the original function to the gradient of the approximation:
\begin{align}\label{eq:GradError}
\|&\nabla f(x) - \nabla \tf(x)\| \nonumber \\
&\eq \| \nabla f(x) -\left(\nabla f(x_0) +H(x_0)(x-x_0)\right)\|\nonumber \\
&\eq \left\|\int_{s=0}^1\left[H(x_0 +s(x-x_0))-H(x_0)\right] ds(x-x_0)\right\| \nonumber\\
&\lee \int_{s=0}^1\rho s ds \| x-x_0\|^2 \nonumber\\
&\eq \frac{\rho}{2}\| x-x_0\|^2~,
\end{align}
the above uses the $\rho$-Lipschitzness of the Hessian.

\paragraph{Part 1:} Here we show that the value of the objective does not rise by more than $\tO(\eta^{3/2})$ in the first $\tO(\eta^{-1/2})$ rounds.
First note the following two lemmas showing that the magnitude of $x_t$'s components in directions with positive (resp. negative) eigenvalues do not increase (resp. decrease) by much.
\begin{lemma}\label{lem:PosEigLemmaGeneral}
Let $T = \tO(\eta^{-1/2})$. If $\lambda_i \geq 0$ then for any $t\in[T]$:
\begin{align*}\label{eq:InductCases}
|\alphat^{(i)}| 
\lee 
&\max\{\eta, |\alphaz^{(i)} | \}+\tO(\eta)\\
+&
\begin{cases}
 \eta &\quad \text{if $\lambda_i |\alphaz^{(i)}| \geq \frac{\rho(\eta T)^2}{2}$ } \\ 
 \eta\min\{1+\frac{\rho\eta T^2}{2 \lambda_i},T \}	&\quad \text{otherwise }
\end{cases}
\end{align*} 
\end{lemma}
\begin{lemma}\label{lem:NegEigLemmaGeneral}
Let $T\leq \tO(\eta^{-1/2})$. If $\lambda_i \leq 0$ then  there exist $c\in [0,1]$ such that  $\forall t\in[T]$:
\begin{equation*}
|\alphat^{(i)}| 
\gee 
\begin{cases}
 |\alphaz^{(i)}-\tO(\eta)|  &\quad \text{if $\lambda_i |\alphaz^{(i)}| > \frac{\rho(\eta T)^2}{2}$ } \\ 
\left| |\alphaz^{(i)}| -  c\eta T-\tO(\eta)\right|	&\quad \text{otherwise }
\end{cases}
\end{equation*} 
\end{lemma}
The following two corollaries of Lemmas~\ref{lem:PosEigLemmaGeneral},~\ref{lem:NegEigLemmaGeneral} show that the objective's 
value does not rise beyond $\tO(\eta^{3/2})$ within $\tO(\eta^{-1/2})$ rounds.
\begin{corollary}\label{cor:LambPositive_General}
Let $2\leq T=\tO(\eta^{-1/2})$, and assume that $\|g_0 \|\leq \beta\sqrt{\eta}$.
Then if $\lambda_i>0$ and we start in a saddle,   the following holds for all $t\in[T]$:
\begin{align*} 
\lambda_i \left( (\alphat^{(i)})^2 - (\alphaz^{(i)})^2 \right) 
\lee  \tO(\eta^{3/2})~.
\end{align*}
\end{corollary}
\begin{corollary}\label{cor:LambNeg_General}
Let $1\leq T\leq  \tO(\eta^{-1/2})$.
Then if $\lambda_i\leq 0$ and we start in a saddle,  
 the following holds for all $t\in[T]$:
\begin{align*} 
|\lambda_i| \left( (\alphat^{(i)})^2 - (\alphaz^{(i)})^2 \right) 
\gee -\tO( \eta^{3/2} )~.
\end{align*}
\end{corollary}
\paragraph{Part 2:} Here we show that the objective's value decreases by $\tOmega(\eta)$ within the first $\tO(\eta^{-1/2})$
rounds.
First note the next lemma showing that the norm of the gradient rises beyond $2\beta \sqrt{\eta}$ within  $\tO(\eta^{-1/2})$ rounds.
\vspace{2pt}
\begin{lemma}\label{lem:Saddle_normGrows_General}
With a probability$\geq1-\xi$, the norm of the gradients rises beyond $2\beta\sqrt{\eta}$ within less than $\tO(\eta^{-1/2})$  steps.
\end{lemma}

Next we show that whenever $\| g_t\|\geq 2\beta\sqrt{\eta}$ then we have improved by value.

\begin{lemma}\label{lem:ImproveLem_General}
Suppose we are in a saddle, and $\|g_0\|\leq \beta\sqrt{\eta}$ then for the first $t$ such that $\| g_t\|\geq 2\beta\sqrt{\eta}$, the following holds w.p.$\geq 1-\xi$
$$f(x_t) \lee  f(x_0) -  \Omega(\eta)~. $$
\end{lemma}
We are now ready to prove Lemma~\ref{lem:SaddleGeneralMain}
\begin{proof}[Proof of Lemma~\ref{lem:SaddleGeneralMain}]
The Lemma follows directly by Lemmas~\ref{lem:Saddle_normGrows_General},~\ref{lem:ImproveLem_General}.
\end{proof}

In Appendix~\ref{Appendix_Saddle}  we provide the complete proofs for the statements that  appear in this section.

\section{Experiments}\label{sec:Experiments}
In many challenging machine learning tasks, first and second order moments are not sufficient in order to extract the underlying parameters of the problem; and higher order moments are required. Such problems include
Gaussian Mixture Models (GMMs), Independent Component Analysis (ICA), Latent Dirichlet Allocation (LDA), and more.
Tensors of order greater than $2$ may capture such high order statistics, and their decomposition enables to reveal the underlying parameters of the problem (see \cite{anandkumar2014tensor} with references therein). Thus, tensor decomposition methods have been extensively studied over the years \cite{harshman1970foundations,kolda2001orthogonal,anandkumar2014tensor}.
While most studies of tensor decomposition methods have focused on the offline setting,  \cite{ge} recently proposed a new setting of online tensor decomposition which is more appropriate for big data tasks. 

Tensor decomposition is an intriguing multi-modal optimization task, which provably acquires
many saddle points. Interestingly, every local minimum is also a global minimum for this task. Thus we decided to focus our experimental study on this task.
  
In what follows, we briefly review tensors and the online decompositions task. Then we present our experimental results comparing our method to  noisy GD, which was proposed in \cite{ge}.
\begin{figure*}[t]
\centering
\subfigure[]{ \label{fig:ExpsLow}
\includegraphics[trim = 15mm 67mm 22mm 67mm, clip,
width=0.31\textwidth ]{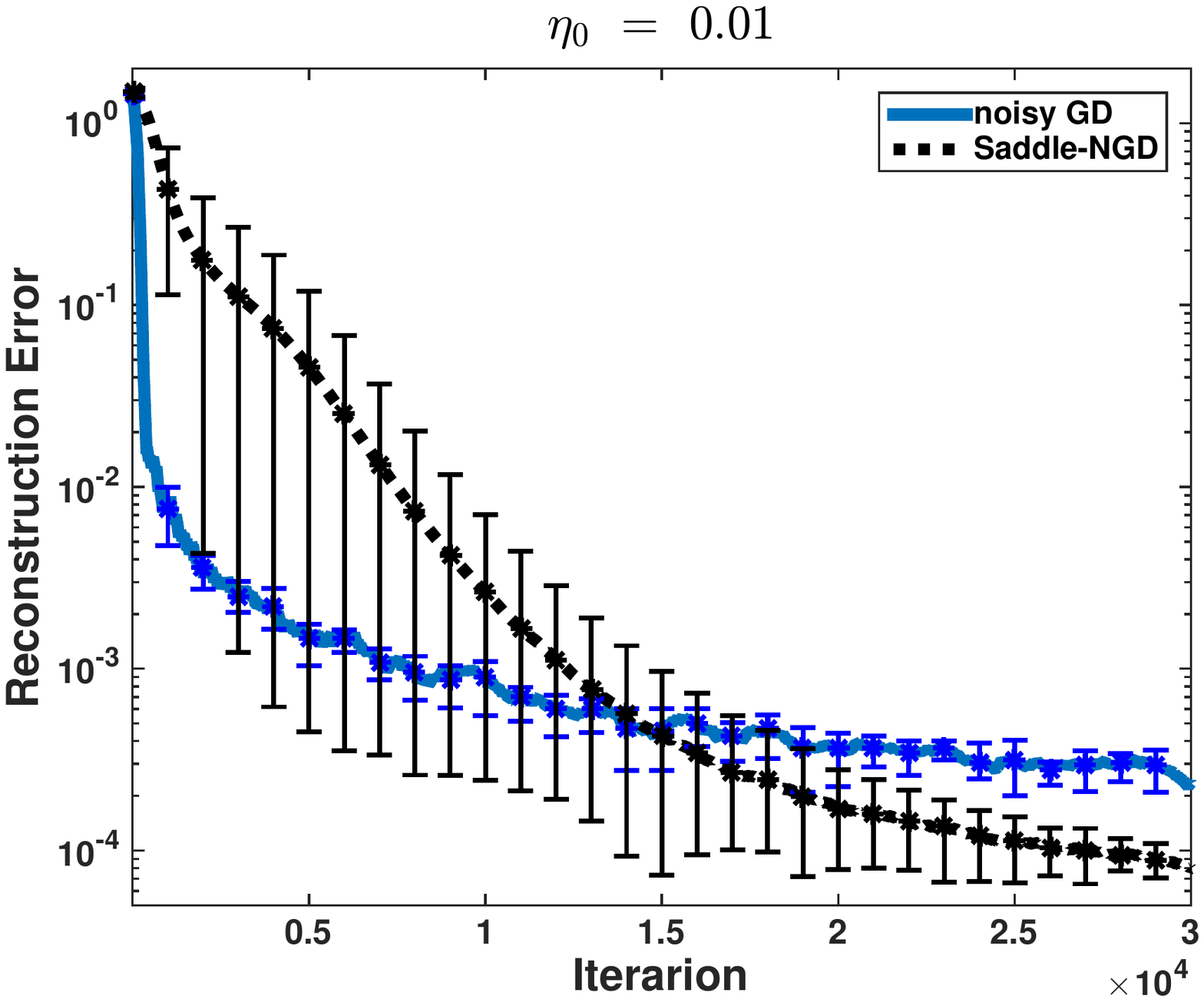}}
\subfigure[]{\label{fig:ExpsMed}
 \includegraphics[trim =  15mm 67mm 27mm 67mm,  clip,
width=0.31\textwidth ]{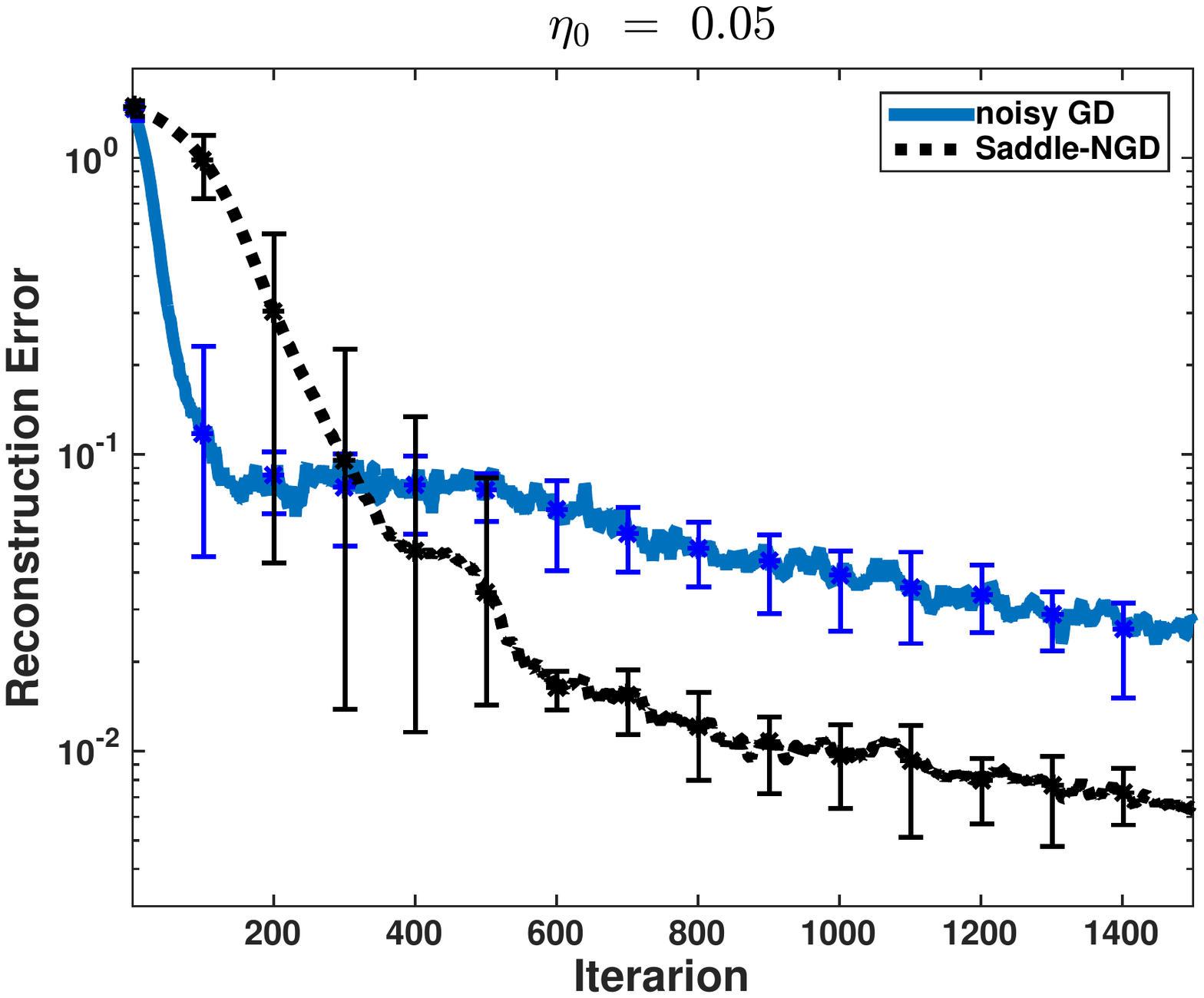}}
 \subfigure[]{ \label{fig:ExpsHigh}
\includegraphics[trim =  15mm 67mm 27mm 67mm,  clip,
width=0.31\textwidth ]{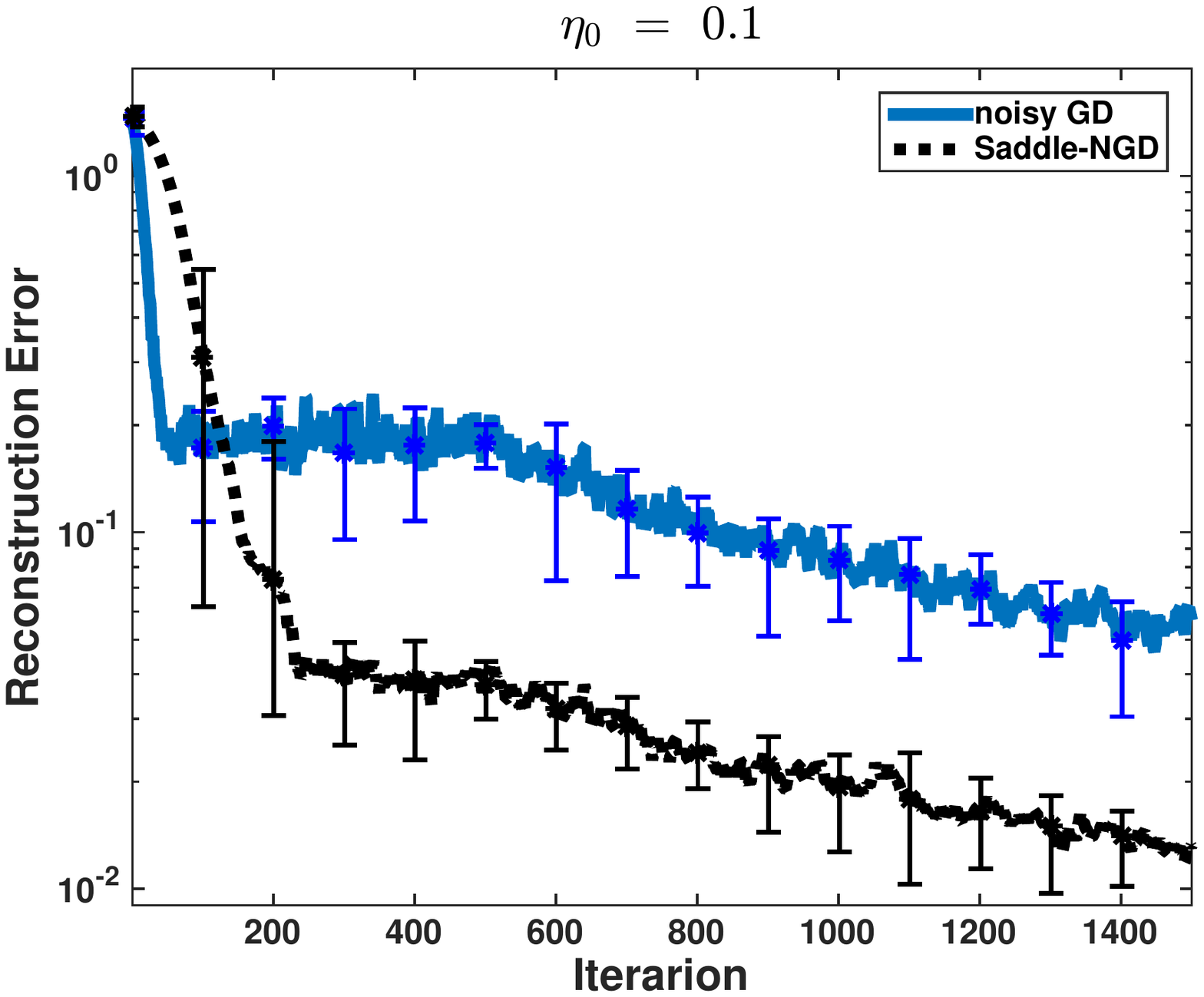}}
\caption{ Noisy-GD Vs. Saddle-NGD for the online ICA problem.} 
\label{fig:ExpsAll}
\end{figure*}

\paragraph{Tensor Decomposition}
A $p$-th order tensor $T\in \btens^p \reals^d$ is a $p$ dimension array. 
Here we focus on $4$-th order tensors $T\in \btens^4 \reals^d$, and use $T_{i_1,i_2,i_3,i_4},\; ( i_1,i_2,i_3,i_4 \in [d])$ to denote the $(i_1,i_2,i_3,i_4)$-th entry of $T$. For a vector $v\in \reals^d$ we use $v^{\tens 4}$ to denote its $4$-th tensor power:
$$(v^{\tens 4})_{i_1,i_2,i_3,i_4} 
\eq 
 v_{i_1}v_{i_2}v_{i_3}v_{i_4} ~.$$
We may represent $T\in  \btens^4 \reals^d$ as a multilinear map, such that  $\forall u,v,z,w \in \reals^d$:
$$T(u,v,z,w) 
\eq 
\sum_{i_1,i_2,i_3,i_4} T_{i_1,i_2,i_3,i_4}  u_{i_1} v_{i_2} z_{i_3} w_{i_4}~.$$
 
We say that $T \in  \btens^4 \reals^d$ has an orthogonal decomposition, if there exists an orthonormal basis $\{ a_i\in \reals^d\}_{i=1}^d $ such that:
$$T \eq \sum_{i=1}^d a_i^{\tens 4} ~.$$
and we call $\{ a_i\}_{i=1}^d$ the components of $T$.
Given $T$ that has an orthogonal decomposition, the offline decomposition problem is to find its components (which are unique up to permutations and sign flips).

\cite{ge} suggest to calculate the components of a decomposable tensor $T$ by minimizing the following \emph{strict-saddle} objective:
\begin{align}\label{eq:TensSaddleObj}
\min_{\forall i, \|u_i\| =1}~\sum_{j\neq i}T(u_i,u_i,u_j,u_j)~.
\end{align}
\paragraph{Online Tensor Decomposition:} In many machine learning tasks,  data often arrives sequentially (sampled from some unknown distribution). When the dimensionality of the problem is large,  it is desirable to store only small batches of the data.
In many such tasks where tensor decomposition is required, data samples $\{z_i\}_i$ arrive from some unknown distribution $\D$. 
And we aim at decomposing a tensor $T$, which is often an expectation over mulitilinear operators $q(z)$, i.e. $T = \E_{z\sim \D}[q(z)]$. Using the linearity of multilinear maps, and the objective appearing in Equation~\eqref{eq:TensSaddleObj}, we can formulate such problems as follows:
\begin{align}\label{eq:TensSaddleObj3}
\min_{\forall i, \|u_i\| =1}\E_{z\sim \D}[\phi_z(u)]
  ~,
\end{align}
 where  $\phi_z(u) \eq \sum_{i \neq  j}q(z)(u_i,u_i,u_j,u_j)$.

\textbf{ICA task:} We discuss a version of the ICA task where we receive samples $y_i: =A x_i\in\reals^d$,  $A\in\reals^{d\times d}$ is an unknown orthonormal linear transformation and   $\{x_i \sim \text{Uni}\{ {\pm1}^d\}\}_i$. Based on the samples $\{y_i\}_i$, 
we aim at reconstructing the matrix $A$ (up to permutation of its columns and sign flips).
\cite{ge} have shown that the ICA problem can be formulated as in Equation~\eqref{eq:TensSaddleObj3}, and demonstrated how to calculate $\nabla \phi_z(u)$. This enables to use the samples in order to produce unbiased gradient estimates of the decomposition objective (and requires to store only small  data minibatches).
\paragraph{ICA Experiments:} We adopt the online tensor decomposition setting for ICA suggested above, and present experimental results comparing our method to the noisy GD method of \cite{ge}.
We take $d=10$, and our performance measure is the reconstruction error defined as:
$$\mathcal{E} \eq \left(\|T - \sum_i u_i^{\tens 4} \|_F\right)^2/\| T\|_F^2~.$$
In our experiments, both methods use a minibatch of size $500$ to calculate gradient estimates. Moreover, both methods employ the following  learning rate rule:
\begin{equation}\label{eq:LR_ruleExperiments}
\eta_t
\eq
\begin{cases}
 \eta_0  &\quad \text{if $t\leq 500$ } \\ 
 \frac{500}{t}\eta_0	&\quad \text{otherwise }
\end{cases}
~.
\end{equation} 
In Figure~\ref{fig:ExpsAll} we present our results\footnote{Note that we have repeated each experiment $10$ times. Figure~\ref{fig:ExpsAll} presents the reconstruction  error averaged over these 10 runs, as well as error bars},
for three values of initial learning rates $\eta_0 \in \{0.01,0.05,0.1\}$.
As can be seen in all three cases, noisy-GD obtains a faster initial improvement. Yet, at some point the error obtained by Saddle-NGD decreases sharply, and eventually Saddle-NGD outperforms noisy-GD.
We found that this behaviour  persisted  when we employed  learning rate rules other  than~\eqref{eq:LR_ruleExperiments}.\footnote{This behaviour also persisted when have employed several noise injection magnitudes to both algorithms.}
Note that for $\eta_0\in\{0.05,0.1 \}$ Saddle-NGD outperforms noisy-GD within $\approx 400$ iterations, for $\eta=0.01$ this only occurs after $\approx 2\cdot 10^4$ iterations. 

\section{Discussion}
We have demonstrated both empirically and theoretically the benefits of using normalized gradients rather than gradients for an intriguing family of non-convex  objectives.
It is natural to ask what are the limits of the achievable rates for evading saddles. Concretely, we ask wether we can  do better than Saddle-NGD using only first order information. 

\section*{Acknowledgement}
I would like to thank Elad Hazan for many useful discussions during the early stages of this work.
\bibliographystyle{abbrvnat}
\bibliography{bib}

\begin{thebibliography}{21}
\providecommand{\natexlab}[1]{#1}
\providecommand{\url}[1]{\texttt{#1}}
\expandafter\ifx\csname urlstyle\endcsname\relax
  \providecommand{\doi}[1]{doi: #1}\else
  \providecommand{\doi}{doi: \begingroup \urlstyle{rm}\Url}\fi

\bibitem[Agarwal et~al.(2016)Agarwal, {Allen-Zhu}, Bullins, Hazan, and
  Ma]{AABHM2017}
N.~Agarwal, Z.~{Allen-Zhu}, B.~Bullins, E.~Hazan, and T.~Ma.
\newblock {Finding Approximate Local Minima for Nonconvex Optimization in
  Linear Time}.
\newblock \emph{ArXiv e-prints}, abs/1611.01146, Nov. 2016.

\bibitem[Allen-Zhu and Hazan(2016)]{allen2016variance}
Z.~Allen-Zhu and E.~Hazan.
\newblock Variance reduction for faster non-convex optimization.
\newblock In \emph{Proceedings of the 33rd International Conference on Machine
  Learning}, 2016.

\bibitem[Anandkumar et~al.(2014)Anandkumar, Ge, Hsu, Kakade, and
  Telgarsky]{anandkumar2014tensor}
A.~Anandkumar, R.~Ge, D.~Hsu, S.~M. Kakade, and M.~Telgarsky.
\newblock Tensor decompositions for learning latent variable models.
\newblock \emph{Journal of Machine Learning Research}, 15:\penalty0 2773--2832,
  2014.

\bibitem[Bengio(2009)]{bengio2009learning}
Y.~Bengio.
\newblock Learning deep architectures for {AI}.
\newblock \emph{Foundations and trends in Machine Learning}, 2\penalty0
  (1):\penalty0 1--127, 2009.

\bibitem[Burges et~al.(2005)Burges, Shaked, Renshaw, Lazier, Deeds, Hamilton,
  and Hullender]{ranking}
C.~Burges, T.~Shaked, E.~Renshaw, A.~Lazier, M.~Deeds, N.~Hamilton, and
  G.~Hullender.
\newblock Learning to rank using gradient descent.
\newblock In \emph{Proceedings of the 22nd international conference on Machine
  learning}, pages 89--96. ACM, 2005.

\bibitem[Choromanska et~al.(2015)Choromanska, Henaff, Mathieu, Ben~Arous, and
  LeCun]{choromanska2015loss}
A.~Choromanska, M.~Henaff, M.~Mathieu, G.~Ben~Arous, and Y.~LeCun.
\newblock The loss surfaces of multilayer networks.
\newblock In \emph{Proceedings of the Eighteenth International Conference on
  Artificial Intelligence and Statistics}, pages 192--204, 2015.

\bibitem[Dauphin et~al.(2014)Dauphin, Pascanu, Gulcehre, Cho, Ganguli, and
  Bengio]{BengioDauphin2014identifying}
Y.~N. Dauphin, R.~Pascanu, C.~Gulcehre, K.~Cho, S.~Ganguli, and Y.~Bengio.
\newblock Identifying and attacking the saddle point problem in
  high-dimensional non-convex optimization.
\newblock In \emph{Advances in Neural Information Processing Systems}, pages
  2933--2941, 2014.

\bibitem[Ge et~al.(2015)Ge, Huang, Jin, and Yuan]{ge}
R.~Ge, F.~Huang, C.~Jin, and Y.~Yuan.
\newblock Escaping from saddle points---online stochastic gradient for tensor
  decomposition.
\newblock In \emph{Proceedings of The 28th Conference on Learning Theory},
  pages 797--842, 2015.

\bibitem[Ge et~al.(2016)Ge, Lee, and Ma]{ge2016matrix}
R.~Ge, J.~D. Lee, and T.~Ma.
\newblock Matrix completion has no spurious local minimum.
\newblock \emph{arXiv preprint arXiv:1605.07272}, 2016.

\bibitem[Ghadimi and Lan(2013)]{ghadimi2013stochastic}
S.~Ghadimi and G.~Lan.
\newblock Stochastic first-and zeroth-order methods for nonconvex stochastic
  programming.
\newblock \emph{SIAM Journal on Optimization}, 23\penalty0 (4):\penalty0
  2341--2368, 2013.

\bibitem[Harshman(1970)]{harshman1970foundations}
R.~A. Harshman.
\newblock Foundations of the parafac procedure: Models and conditions for an"
  explanatory" multimodal factor analysis.
\newblock 1970.

\bibitem[Hazan et~al.(2015)Hazan, Levy, and Shalev-Shwartz]{hazan2015beyond}
E.~Hazan, K.~Y. Levy, and S.~Shalev-Shwartz.
\newblock Beyond convexity: Stochastic quasi-convex optimization.
\newblock In \emph{Advances in Neural Information Processing Systems}, pages
  1585--1593, 2015.

\bibitem[Jain et~al.(2010)Jain, Meka, and Dhillon]{jain2010guaranteed}
P.~Jain, R.~Meka, and I.~S. Dhillon.
\newblock Guaranteed rank minimization via singular value projection.
\newblock In \emph{Advances in Neural Information Processing Systems}, pages
  937--945, 2010.

\bibitem[Kiwiel(2001)]{kiwiel2001convergence}
K.~C. Kiwiel.
\newblock Convergence and efficiency of subgradient methods for quasiconvex
  minimization.
\newblock \emph{Mathematical programming}, 90\penalty0 (1):\penalty0 1--25,
  2001.

\bibitem[Kolda(2001)]{kolda2001orthogonal}
T.~G. Kolda.
\newblock Orthogonal tensor decompositions.
\newblock \emph{SIAM Journal on Matrix Analysis and Applications}, 23\penalty0
  (1):\penalty0 243--255, 2001.

\bibitem[Nesterov(1984)]{nesterov1984minimization}
Y.~E. Nesterov.
\newblock Minimization methods for nonsmooth convex and quasiconvex functions.
\newblock \emph{Matekon}, 29:\penalty0 519--531, 1984.

\bibitem[Saad and Solla(1995)]{saad1995exact}
D.~Saad and S.~A. Solla.
\newblock Exact solution for on-line learning in multilayer neural networks.
\newblock \emph{Physical Review Letters}, 74\penalty0 (21):\penalty0 4337,
  1995.

\bibitem[Saxe et~al.(2015)Saxe, McClelland, and Ganguli]{saxe2}
A.~M. Saxe, J.~L. McClelland, and S.~Ganguli.
\newblock Exact solutions to the nonlinear dynamics of learning in deep linear
  neural networks.
\newblock In \emph{ICLR}, 2015.

\bibitem[Sun et~al.(2015)Sun, Qu, and Wright]{sun2015complete}
J.~Sun, Q.~Qu, and J.~Wright.
\newblock Complete dictionary recovery over the sphere.
\newblock In \emph{Sampling Theory and Applications (SampTA)}, pages 407--410.
  IEEE, 2015.

\bibitem[Sun et~al.(2016)Sun, Qu, and Wright]{sun2016geometric}
J.~Sun, Q.~Qu, and J.~Wright.
\newblock A geometric analysis of phase retrieval.
\newblock \emph{arXiv preprint arXiv:1602.06664}, 2016.

\bibitem[Sutton et~al.(1999)Sutton, McAllester, Singh, Mansour, et~al.]{GradRL}
R.~S. Sutton, D.~A. McAllester, S.~P. Singh, Y.~Mansour, et~al.
\newblock Policy gradient methods for reinforcement learning with function
  approximation.
\newblock In \emph{NIPS}, volume~99, pages 1057--1063, 1999.

\end{thebibliography}

\newpage
\appendix
\section{Proof of Theorem~\ref{theorem:MAinPaper}} \label{Proof_theorem:MAinPaper}
\begin{proof}[Proof of Theorem~\ref{theorem:MAinPaper}]
We will now show that starting at any $x_0\in \reals^d$,  within $\tO(\eta^{-3/2})$ rounds then w.p.$\geq 1/2$ we arrive $\tO(\sqrt{\eta})$ close to a local minimum.  Since this is true for any starting point, 
repeating this for $\Theta(\log(1/\xi))$ epochs
 implies that  within $\tO(\eta^{-3/2})$ rounds then w.p.$\geq 1-\xi$ we reach a point that is $\tO(\sqrt{\eta})$ close to a local minimum.

Let us define the following sets with respects to the three scenarios of the strict saddle property:
\begin{align*}
&\A_1\eq \{x:\; \|\nabla f(x)\|\geq \beta\sqrt{\eta}\} \\
&\A_2 \eq \{x:\; \|\nabla f(x)\|< \beta\sqrt{\eta} \text{ and } \lambda_{min}(\nabla^2 f(x))\leq -\gamma\} \\
&\A_3 \eq  \A_1^c \cap A_2^c
\end{align*}
By the strict saddle property, choosing $\eta$ small enough  ensures that all points in $\A_3$ are $\tO(\sqrt{\eta})$ close to a local minimum. In what follows we show that with a high probability we reach a point in $\A_3$ within $\tO(\eta^{-3/2})$ rounds.

Define the following sequence of stopping times with $\tau_0=0$, and
\begin{equation}\nonumber
\tau_{i+1}
\eq
\begin{cases}
	\tau_i+1	&\quad x_{\tau_i}\in \A_1\cup \A_3  \\ 
	\tau_i+T(x_{\tau_i})         	&\quad x_{\tau_i}\in \A_2  \\  
\end{cases}
\end{equation}
Where given $x\in\A_2$ then $T(x)$ is defined to be the first time after  that Saddle-NGD reaches a point with a value lower by at least $\tOmega(\eta)$ than $f(x)$.
Define a sequence of filtrations $\{ \F_t \}_{t}$, as follows, $\F_t=\{x_0,\ldots,x_t\}$, and 
note that according to Lemma~\ref{lem:SaddleGeneralMain} and Lemma~\ref{lem:LargeGrads} the following holds:
\begin{align}
&\E[f(x_{\tau_{i+1}})-f(x_{\tau_i})\mid x_{\tau_i}\in\A_1,\F_{\tau_i}]
\lee
 -\tOmega(\eta^{3/2}) \label{eq:Main1}\\
&\E[f(x_{\tau_{i+1}})-f(x_{\tau_i})\mid x_{\tau_i}\in\A_2,\F_{\tau_i}]
\lee
 -\tOmega(\eta)\label{eq:Main2}
\end{align}
The second inequality holds  by the following Corollary of  Lemma~\ref{lem:SaddleGeneralMain}
\begin{corollary}
Suppose that $\|g_0\|\leq \beta \sqrt{\eta}\leq \nu$ and we are in a saddle point, meaning that $\lambda_{min}(\nabla^2 f(x_0))\leq -\gamma$.
And consider the Saddle-NGD algorithm performed over a strict-saddle function $f$. Then within  $t \leq \tO( \eta^{-1/2})$ steps, we will have :
$$\E[f(x_{t})-f(x_0)\mid x_0\in\A_2, \F_0] 
\lee  
 -  \tOmega(\eta)~. $$
\end{corollary}
Note that according to the above corollary we have $T(x)\leq \tO(\eta^{-1/2}), \;\forall x\in\A_2$.
Combining Equations~\eqref{eq:Main1},\eqref{eq:Main2} the following holds:
\begin{align}
&\E[f(x_{\tau_{i+1}})-f(x_{\tau_i})\mid x_{\tau_i}\notin \A_3]
\lee 
-(\tau_{i+1}-\tau_i)\tOmega(\eta^{3/2}) \label{eq:Main3}~.
\end{align}
Now define a series of events $\{ E_i\}_i$  as follows $E_i = \{\exists j\leq i: x_{\tau_j}\in\A_3 \}$.
Note that $E_{i-1}\subset E_{i}$ and therefore $P(E_{i-1})\leq P(E_i)$. 
Now consider $f(w_{\tau_{i+1}})1_{E_i}$
\begin{align}
\E&[f(x_{\tau_{i+1}})1_{E_i}-f(x_{\tau_i})1_{E_{i-1}}]  \nonumber\\
&\eq
 \E[f(x_{\tau_{i+1}})1_{E_{i-1}}-f(x_{\tau_i})1_{E_{i-1}}]+\E[f(x_{\tau_{i+1}})(1_{E_i}-1_{E_{i-1}})] \nonumber\\
&\eq 
\E[f(x_{\tau_{i+1}})-f(x_{\tau_i})]+\E[f(x_{\tau_{i+1}})-f(x_{\tau_i})\mid E_i^c]P(E_i^c)
+B(P(E_i)-P(E_{i-1}))\nonumber\\
&\lee
 \E[f(x_{\tau_{i+1}})-f(x_{\tau_i})]-(\tau_{i+1}-\tau_i)\Omega(\eta^{3/2})P(E_i^c)
+B(P(E_i)-P(E_{i-1}))~.
\end{align}
Summing the above equation over $i$ we conclude that:
\begin{align*}
\E&[f(x_{\tau_{i+1}})1_{E_i}]-f(x_0)  \\
&\lee 
\E[f(x_{\tau_{i+1}})-f(x_{0})]-\tau_{i+1}\tOmega(\eta^{3/2})P(E_i^c)
+B~.
\end{align*}
Since $|f(x)|\leq B,\; \forall x$ we conclude that setting $\tau_{i+1}\geq \tOmega(\eta^{-3/2})$ then  $P(E_i)\geq 1/2$.
Since this is true for any starting point then performing  the Saddle-NGD procedure for $\tO(\eta^{-3/2} )$
rounds ensures that we arrive $\tO(\sqrt{\eta})$ close to a local minimum ($\A_3$) at least once.  Lemma~\ref{lem:LocalOptFull} ensures that once we arrive at a local minimum we remain at its $\tO(\sqrt{\eta})$ proximity
w.p.$\geq 1-\xi$. Finally, since $f$ is $\beta$-smooth, this $\tO(\sqrt{\eta})$ proximity  implies that we reach a point which is $\tO(\eta)$ close by value to $f(x^*)$-the value of the local minimum.
\end{proof}
\section{ Proof of Lemma~\ref{lem:LocalOptFull} (Local Minimum)}\label{sec:Omitted Proof Local Minimum}
Here we  prove  Lemma~\ref{lem:LocalOptFull} regarding the local minimum for the general case which includes the noisy updates. In Section~\ref{sec:Proof_lem:StrCvxity} we prove Lemma~\ref{lem:StrCvxity} which is used during the proof of Lemmas~\ref{lem:LocalOptFull},\ref{lem:LocalOpt}.

\begin{proof}[Proof of Lemma~\ref{lem:LocalOptFull}]
Recall that according to Algorithm~\ref{algorithm:SNGD}, once we had a noisy update, we use noiseless updates for the following $\Omega(\eta^{-1/2})$ rounds.  
Now, due to the local strong convexity 
we know that $\|x_0 -x^*\| \leq \frac{1}{\alpha}\|g_0\| \leq \frac{\beta}{\alpha} \sqrt{\eta}$. 
We will now show that if the distance of $x_1$ from the local minimum $x^*$ increases beyond $\frac{\beta}{\alpha}\sqrt{\eta}$ due to the noisy update at round $t=0$, then this distance will be decreased below $\frac{\beta}{\alpha}\sqrt{\eta}$ in the following $\tilde{\Theta}(\eta^{-1/2})$ rounds of noiseless updates. 

Let us first bound the increase in the  square distance to $x^*$ due to the first noisy update. 
\begin{align}\label{eq:localMin2}
\|x_{1}-x^*\|^2 &
\eq
 \|x_{0} -x^* \|^2 -2(\eta \hat{g}_0+\theta n_0)^\top (x_{0}-x^*) 
+\| \eta \hat{g}_0 + \theta n_0 \|^2 \nonumber \\
&\lee
 \|x_{0} -x^* \|^2 + \tO(\eta^{3/2}\sqrt{d})+ \tO(\eta^2 d)\nonumber \\
&\lee
 \|x_{0} -x^* \|^2 +\eta~.
\end{align}
Here we used $\hatg_0(x_0-x^*)\geq 0$ which follows by local strong-convexity around $x^*$, we also used   $\theta n_0\sim \tO(\eta)\N(0,\I_d)$, which implies that w.p.$\geq1-\xi$ we have $\|\theta n_0 \|\leq \tO(\eta\sqrt{d}\log(1/\xi))$. We also used $\|x_0 -x^*\| \leq \O( \sqrt{\eta})$. The last inequality follows by choosing $\eta\leq \tO(1/d)$.

Let $t\geq 2$, and suppose  $\|x_{t-1}-x^*\|\geq \frac{\beta}{\alpha}\sqrt{\eta}$, the Saddle-NGD update rule implies:
\begin{align}\label{eq:localMin}
\|x_{t}-x^*\|^2 
&\eq
 \|x_{t-1} -x^* \|^2 -2\eta \hat{g}_t^\top (x_{t-1}-x^*) + \eta^2 \nonumber \\
& \eq
  \|x_{t-1} -x^* \|^2 -2\eta \frac{1}{\|g_t\|} {g}_t^\top (x_{t-1}-x^*) + \eta^2 \nonumber \\
&\lee
  \|x_{t-1} -x^* \|^2 -2\eta \frac{1}{\|g_t\|} \alpha \|x_{t-1}-x^*\|^2 + \eta^2    &\mbox{Lemma~\ref{lem:StrCvxity}}\nonumber \\
&\lee
 \|x_{t-1} -x^* \|^2 -2\eta\frac{\alpha}{\beta}  \|x_{t-1}-x^*\| + \eta^2  &\mbox{smoothness}\nonumber \\
&\lee
 \|x_{t-1} -x^* \|^2-\eta\sqrt{\eta} &\mbox{$\|x_{t-1}-x^*\|\geq\frac{\beta}{\alpha}\sqrt{\eta}$}
\end{align}
here in the first inequality we use Lemma~\ref{lem:StrCvxity},  the second inequality uses $\|g_t\|\leq \beta \|x_t-x^* \|$ which follows from smoothness, and the last inequality uses $\|x_{t-1}-x^*\|\geq\frac{\beta}{\alpha}\sqrt{\eta}$.

Combining Equations~\eqref{eq:localMin2},\eqref{eq:localMin} it follows that within the total of 
$\Theta(\eta^{-1/2})$ rounds of noiseless updates the distance to the local minimum must decrease to $\frac{\beta}{\alpha}\sqrt{\eta}$.
The rest of the proof goes along the same lines as the proof of Lemma~\ref{lem:LocalOpt}.
\end{proof}
\subsection{Proof of Lemma~\ref{lem:StrCvxity}}\label{sec:Proof_lem:StrCvxity}
\begin{proof}
Writing the strong-convexity inequality for $F$ at both $x,x^*$, and using   $\nabla F(x^*)^\top(x-x^*)\geq 0$, which follows by the optimality of $x^*$,  we have:
\begin{align}
F(x^*) -F(x)&\gee \nabla F(x)^\top(x^*-x) + \frac{\alpha}{2} \| x-x^*\|^2\\
F(x) -F(x^*)&\gee   \frac{\alpha}{2} \| x-x^*\|^2
\end{align}
summing the above equations the lemma follows.
\end{proof}

\section{ Omitted Proofs from Section~\ref{sec:SaddleMetaSection} (Saddle Analysis) }
\label{Appendix_Saddle}
\subsection{Proof of Lemma~\ref{lem:PosEigLemmaGeneral}} \label{Proof_lem:PosEigLemmaGeneral}
\begin{proof}
Since  $T=\tO(\eta^{-1/2})$, and Saddle-NGD perform noisy updates once every $N_0=\Theta(\eta^{-1/2})$ rounds, it follows that the total increase in $|\alpha_{t}^{(i)}|$ due to noisy rounds is bounded by
 $\theta \frac{T}{N_0} = \tO(\eta)$. Thus in the rest of the proof we assume noiseless  updates, and show the following to hold:
 \begin{equation}\label{eq:InductCasesProof}
|\alphat^{(i)}| 
\lee
\max\{\eta, |\alphaz^{(i)} | \}+
\begin{cases}
 \eta &\quad \text{if $\lambda_i |\alphaz^{(i)}| \geq (\rho/2)\eta^2 T^2$ } \\ 
 \eta\min\{1+\frac{\rho\eta T^2}{2 \lambda_i},T \}	&\quad \text{otherwise }
\end{cases}
~.
\end{equation}

According to Equation~\eqref{eq:GradError} then for every $x_t\in \reals^d$ we have:
\begin{align*}
g_t \eq \nabla f(x_t) \eq \nabla \tf(x_t)+\Deltat \eq H_0 x_t + \Deltat~,
\end{align*}
here we used $\tf(x) =0.5x^\top H_0 x$, and the notation  $\Deltat:=\nabla f(x_t) - \nabla \tf(x_t)$. Thus, combining Equation~\eqref{eq:GradError}, together with  $\| x_t-x_0\|\leq \eta T,\; \forall t\in[T]$(which follows by the Saddle-NGD update rule), we obtain:
\begin{align}\label{eq:DeltaUpperBound}
\|\Deltat \|
\lee (\rho/2)\eta^2 T^2, \quad \forall t\in[T] ~.
\end{align}
Denote $\Deltat^{(i)} = e_i^\top \Deltat$, then coordinate-wise the NGD rule translates to 
\begin{align*}
\alphatp^{(i)} & 
\eq 
\alphat^{(i)} -\eta \frac{g_t^{(i)}}{\| g_t\|} \\
&\eq
\alphat^{(i)} -\eta \frac{\lambda_i \alphat^{(i)} + \Deltat^{(i)}}{\| g_t\|}~.
\end{align*}
\textbf{First part:} First we show that the following always applies:
\begin{equation}\label{eq:InductCases1}
|\alphat^{(i)}| \lee 
 \max\{\eta, |\alphaz^{(i)} | \}+ \eta\min\{1+\frac{\rho\eta T^2}{2 \lambda_i},T \}~.	
 \end{equation} 
We will prove Equation~\eqref{eq:InductCases1} by showing the following  to hold for all $t\in[T]$:
\begin{equation}\label{eq:InductCases}
|\alphat^{(i)}|
\lee 
\max\{\eta, |\alphaz^{(i)} |\}+\eta
\begin{cases}
 1+\frac{\rho\eta T^2}{2 \lambda_i} &\quad \text{if $1+\frac{\rho\eta T^2}{2 \lambda_i}\leq T$ } \\ 
 t	&\quad \text{otherwise }
\end{cases}
~.
\end{equation} 
We will now prove Equation~\eqref{eq:InductCases}  by induction. The base case $t=0$  clearly holds. 
Now assume by the induction assumption that it holds for $t\in[T]$, then we divide into two cases:

\emph{Case 1:}
Suppose that $T\leq 1+\frac{\rho\eta T^2}{2 \lambda_i}$.
Since $\alphat^{(i)}$ can not change by more than $\eta$ in each round then the following holds:
\begin{align*}
|\alphatp^{(i)}| 
\lee
 |\alphat^{(i)}|+\eta  
 \lee
  \max\{\eta, |\alphaz^{(i)} |\}+\eta (t+1)~.
\end{align*}
thus the induction hypothesis of Equation~\eqref{eq:InductCases} holds in this case.

\emph{Case 2:}
Suppose that $T\geq 1+\frac{\rho\eta T^2}{2 \lambda_i}$.
In this case  our induction hypothesis asserts:
$$|\alphat^{(i)}| 
\lee 
\max\{\eta, |\alphaz^{(i)} |\}+\eta(1+\frac{\rho\eta T^2}{2 \lambda_i})~.$$
If $|\alphat^{(i)}| \geq \max\{\eta, |\alphaz^{(i)} |\}+\frac{\rho\eta^2 T^2}{2 \lambda_i}$ then using Equation~\eqref{eq:DeltaUpperBound}, we conclude that:
$$\sign(g_t^{(i)}) 
\eq
 \sign(\lambda_i \alphat^{(i)} + \Deltat^{(i)}) 
\eq
\sign(\alphat^{(i)})~.$$
The above implies:
\begin{align}\label{eq:BasicBound}
|\alphatp^{(i)}| 
&\eq
 \left| \alphat^{(i)} -\eta \frac{g_t^{(i)}}{\| g_t\|}\right| \nonumber\\
&\eq
\left| |\alphat^{(i)}| - \eta \frac{|g_t^{(i)}|}{\| g_t\|}\right|   &\mbox{$\sign(\alphat^{(i)})=\sign(g_t^{(i)})$}  \nonumber\\
&\lee
 \max\left\{ \eta \frac{|g_t^{(i)}|}{\| g_t\|}, |\alphat^{(i)}|\right\}   \nonumber\\
&\lee
  \max\{\eta, |\alpha^{(i)}_t| \}                      &\mbox{$0\leq \frac{|g_t^{(i)}|}{\| g_t\|}\leq 1$}      \nonumber\\
 &\lee
   \max\{\eta, |\alphaz^{(i)} |\}+\eta (t+1)~, &\mbox{Induction hypothesis} 
\end{align}
 thus by Equation~\eqref{eq:BasicBound}, the induction hypothesis holds.
 
If $|\alphat^{(i)}| \leq \max\{\eta, |\alphaz^{(i)} |\}+\frac{\rho\eta^2 T^2}{2 \lambda_i}$, then since each coordinate does not change more than $\eta$ in each iterartion
then we have:
\begin{align*}
|\alphatp^{(i)}| 
\lee
 |\alphat^{(i)}|+\eta 
 \lee
  \max\{\eta, |\alphaz^{(i)} |\}+\frac{\rho\eta^2 T^2}{2 \lambda_i}+\eta~,
\end{align*}
and again, the induction hypothesis holds.

\textbf{Second part:} Here we show that whenever $\lambda_i |\alphaz^{(i)}| \geq (\rho/2)\eta^2 T^2$ then the following applies:
\begin{equation}\label{eq:InductCases2}
|\alphat^{(i)}|
\lee 
\max\{\eta, |\alphaz^{(i)} |\}+\eta
~.
 \end{equation} 
We will now show by induction that Equation~\eqref{eq:InductCases2} holds for any $t\in[T]$.
The base case for $t=0$ clearly holds.
Now assume by the Induction Hypothesis that $|\alphat^{(i)}|\leq \max\{\eta, |\alphaz^{(i)} |\}+\eta$. We now divide into two cases:

\emph{Case 1:}
Suppose that $|\alphat^{(i)}|\leq \max\{\eta, |\alphaz^{(i)} |\}$. Since each coordinate does not change by more than $\eta$ in each round, then:
$$|\alphatp^{(i)}|
\lee
 |\alphat^{(i)}| +\eta 
 \lee
  \max\{\eta, |\alphaz^{(i)} |\}+\eta~,$$
thus, the induction hypothesis holds.

\emph{Case 2:} 
Suppose that $|\alphat^{(i)}|\geq \max\{\eta, |\alphaz^{(i)} |\}$. Recall that $t\in[T]$, and $|\Delta_t^{(i)}| \leq (\rho/2)\eta^2 T^2$, thus:
\begin{align*}
\lambda_i |\alphat^{(i)}|
&\gee
 \lambda_i |\alphaz^{(i)} |\\
&\gee
 (\rho/2)\eta^2 T^2 \\
&\gee
 |\Delta_t^{(i)}|~.
\end{align*}
The above implies that:
$$\sign(g_t^{(i)}) \eq \sign(\lambda_i \alphat^{(i)} + \Deltat^{(i)}) 
\eq \sign(\alphat^{(i)})$$
In this case, a similar analysis to the one appearing in Equation~\eqref{eq:BasicBound}, shows:
\begin{align*}
|\alphatp^{(i)}| 
\lee
 \max\{ \eta ,|\alphat^{(i)}| \}  \leq \max\{\eta, |\alphaz^{(i)} |\}+\eta~,
\end{align*}
 and the induction hypothesis holds.

Combining Equations~\eqref{eq:InductCases1},\eqref{eq:InductCases2}, establishes the Lemma.
\end{proof}

\subsection{Proof of Lemma~\ref{lem:NegEigLemmaGeneral}} \label{Proof_lem:NegEigLemmaGeneral}
\begin{proof}
Since  $T=\tO(\eta^{-1/2})$, and Saddle-NGD perform noisy updates once every $N_0=\Theta(\eta^{-1/2})$ rounds, it follows that the total decrease in $|\alpha_{t}^{(i)}|$ due to noisy rounds is bounded by
 $\theta \frac{T}{N_0} = \tO(\eta)$. Thus in the rest of the proof we assume noiseless  updates, and show the following to hold for some $c\in[0,1]$:
 \begin{equation}\label{eq:NegInductCasesNoiselessProof}
|\alphat^{(i)}| 
\gee
\begin{cases}
 |\alphaz^{(i)}|  &\quad \text{if $\lambda_i |\alphaz^{(i)}| > (\rho/2)\eta^2 T^2$ } \\ 
\left| |\alphaz^{(i)}| -  c\eta T\right|	&\quad \text{otherwise }
\end{cases}
~.
\end{equation} 
According to Equation~\eqref{eq:GradError} then for any $x_t\in \reals^d$ we have:
\begin{align*}
g_t \eq \nabla f(x_t) \eq \nabla \tf(x_t)+\Deltat \eq H_0 x_t + \Deltat ~,
\end{align*}
here we used $\tf(x) =0.5x^\top H_0 x$, and the notation  $\Deltat:=\nabla f(x_t) - \nabla \tf(x_t)$. Thus, combining Equation~\eqref{eq:GradError}, together with  $\| x_t-x_0\|\leq \eta T,\; \forall t\in[T]$ (which follows by the Saddle-NGD update rule), we obtain:
\begin{align*}
\|\Deltat \|
\lee 
(\rho/2)\eta^2 T^2, \quad \forall t\in[T] 
\end{align*}
Denote $\Deltat^{(i)} = e_i^\top \Deltat$, then coordinate-wise the NGD rule translates to 
\begin{align*}
\alphatp^{(i)} 
& \eq
 \alphat^{(i)} -\eta \frac{g_t^{(i)}}{\| g_t\|} \\
&\eq
\alphat^{(i)} -\eta \frac{\lambda_i \alphat^{(i)} + \Deltat^{(i)}}{\| g_t\|}
\end{align*}
\textbf{First part:} First we show that the following always applies for some $c\in[0,1]$:
\begin{equation}\label{eq:NegInductCases1}
|\alphat^{(i)}| 
\gee 
 \left| |\alphaz^{(i)}| -  c\eta T 	\right|~.
\end{equation} 
Equation~\eqref{eq:NegInductCases1} follows since  by the Saddle-NGD update rule the magnitude of the query points   can not decrease by more than $\eta T$ over $T$ rounds.

\textbf{Second part:} Here we show that the following applies whenever  $\lambda_i |\alphaz^{(i)}| > (\rho/2)\eta^2 T^2$:
\begin{equation}\label{eq:NegInductCases2}
|\alphat^{(i)}| \gee  |\alphaz^{(i)}|~. 	
\end{equation} 
We will prove Equation~\eqref{eq:NegInductCases2}, by induction. Clearly the base case $t=0$ holds.
Now assume that $|\alphat^{(i)}| \geq  |\alphaz^{(i)}| > (\rho/2\lambda_i)\eta^2 T^2$. Since $g_t^{(i)} = \lambda_i \alphat^{(i)} + \Deltat^{(i)}$, and $|\Deltat^{(i)}|\leq (\rho/2)\eta^2 T^2$, then we necessarily have:
 $$\sign( g_t^{(i)}) \eq \sign(\lambda_i \alphat^{(i)} )\eq -\sign(\alphat^{(i)})~,$$
 where we used $\lambda_i\leq 0$.
 Hence,
 $$|\alphatp^{(i)}| 
 \eq 
 \left|\alphat - \eta\frac{g_t^{(i)}}{\| g_t\|}\right| 
 \gee 
 |\alphat^{(i)}|~,$$
 and the induction hypothesis holds.
 Combining Equations~\eqref{eq:NegInductCases1}, \eqref{eq:NegInductCases2}, proves the lemma.
\end{proof}

\subsection{Proof of Corollary~\ref{cor:LambPositive_General}} \label{Proof_cor:LambPositive_General}
\begin{proof} 
We prove the corollary using Lemma~\ref{lem:PosEigLemmaGeneral}.  Note that for simplicity we ignore the $\tO(\eta)$ factors appearing in Lemma~\ref{lem:PosEigLemmaGeneral}, considering these factors yields similar guarantees.
We divide the proof into two cases which depend on the size of $|g_0^{(i)}|$. 
\textbf{High $|g_0^{(i)}|$:} Suppose that $|g_0^{(i)}|:= \lambda_i |\alphaz^{(i)}| \geq (\rho/2)\eta^2 T^2$
\begin{align}\label{eq:PosLamVal1}
\lambda_i &\left( (\alphat^{(i)})^2 - (\alphaz^{(i)})^2 \right) \nonumber\\
&\lee
  \lambda_i\left( (\max\{\eta, \alphaz^{(i)}\}+\tO(\eta))^2-(\alphaz^{(i)})^2  \right) \nonumber &\mbox{Lemma~\ref{lem:PosEigLemmaGeneral}}\\
&\lee
  \lambda_i\left( \max\{\eta^2, (\alphaz^{(i)})^2\}-(\alphaz^{(i)})^2  \right)+2\lambda_i |\alphaz^{(i)}|\tO(\eta) + \lambda_i \tO(\eta^2) \nonumber\\
&\lee
 2\beta\eta^2 +\tO( \eta^{3/2})=\tO(\eta^{3/2}) &\mbox{$\lambda_i |\alphaz^{(i)}|  \leq \beta\sqrt{\eta}$}
\end{align}
where the first inequality uses Lemma~\ref{lem:PosEigLemmaGeneral}, 
and the last inequality uses $\lambda_i |\alphaz^{(i)}| = |g_0^{(i)}| \leq \beta\sqrt{\eta}$, and $\beta = \max_i |\lambda_i|$.

\textbf{Low  $|g_0^{(i)}|$:} Suppose that $|g_0^{(i)}|:= \lambda_i |\alphaz^{(i)}| \leq (\rho/2)\eta^2 T^2$.
Denote $B=\min\{ 1+\frac{\rho\eta T^2}{2 \lambda_i},T\}$,  and note that this expression is maximized when
$\lambda_i = \frac{\rho\eta T^2}{2(T-1)}$. 
\begin{align}\label{eq:PosLamVal2}
\lambda_i &\left( (\alphat^{(i)})^2 - (\alphaz^{(i)})^2 \right) \nonumber\\
&\lee
  \lambda_i\left( \max\{\eta^2, (\alphaz^{(i)})^2\}-(\alphaz^{(i)})^2  \right)+
2  \tO(\eta)\max\lrset{\lambda_iB\eta,\lambda_i|\alphaz^{(i)}| B} + \lambda_i B^2 \tO(\eta^2) \nonumber\\
&\lee
 \lambda_i\eta^2 +\tO(\eta) \max\lrset{\frac{\rho\eta^2 T^3}{T-1},\frac{\beta\rho\eta^2 T^3}{2}}+\frac{\rho\eta T^4}{2(T-1)} \tO(\eta^2)\nonumber \\
&\lee
 \beta\eta^2 +\tO((\eta T)^3)\nonumber\\
 &\eq 
 \tO(\eta^{3/2})~,
\end{align}
here in the second inequality we used $\lambda_i B \leq \frac{\rho\eta T^3}{2(T-1)}$, we also used 
$B\leq T$, $|g_0^{(i)}|:= \lambda_i |\alphaz^{(i)}| \leq (\rho/2)\eta^2 T^2$,
and also $\lambda_i B^2 \leq \frac{\rho\eta T^4}{2(T-1)}$.
The third inequality uses $T\geq 2$, and the last inequality uses $T=\tO(\eta^{-1/2})$.

Combining Equations~\eqref{eq:PosLamVal1},\eqref{eq:PosLamVal2}, establishes the lemma.
\end{proof}

\subsection{Proof of Corollary~\ref{cor:LambNeg_General}} \label{proof_cor:LambNeg_General}
\begin{proof}
For simplicity we are going to use the following bound of Lemma~\ref{lem:NegEigLemmaGeneral} which ignores the $\tO(\eta)$ factors in the original lemma (these appear due to the noisy updates):
\begin{equation}\label{eq:NegInductCasesNoiseless}
|\alphat^{(i)}| 
\gee
\begin{cases}
 |\alphaz^{(i)}|  &\quad \text{if $\lambda_i |\alphaz^{(i)}| > (\rho/2)\eta^2 T^2$ } \\ 
\left| |\alphaz^{(i)}| -  c\eta T\right|	&\quad \text{otherwise }
\end{cases}
~,
\end{equation} 
which holds for some $c\in[0,1]$.
Including the original $\tO(\eta)$ factors in the calculations yields an additional factor of 
$-\tO(\eta^{3/2}+\eta^2T) = -\tO(\eta^{3/2})$ (recall $T=\tO(\eta^{-1/2})$).

Using Equation~\eqref{eq:NegInductCasesNoiseless} we divide into two cases:
\emph{Case 1:} Suppose that $\lambda_i |\alphaz^{(i)}| > (\rho/2)\eta^2 T^2$, then by 
Equation~\eqref{eq:NegInductCasesNoiseless}  we have $|\alphat^{(i)}| \geq |\alphaz^{(i)}|$ and therefore 
$$|\lambda_i| \left( (\alphat^{(i)})^2 - (\alphaz^{(i)})^2 \right)
\gee 
0~.$$

\emph{Case 2:} Suppose that $\lambda_i |\alphaz^{(i)}| \leq (\rho/2)\eta^2 T^2$, using
Equation~\eqref{eq:NegInductCasesNoiseless}  we get:
\begin{align*} 
|\lambda_i| \left( (\alphat^{(i)})^2 - (\alphaz^{(i)})^2 \right) 
&\gee
 |\lambda_i| \left( (|\alphaz^{(i)}| -  c\eta T 	)^2 - (\alphaz^{(i)})^2 \right) \\
&\gee
 -2c|\lambda_i \alphaz^{(i)}| \eta T  + 
 |\lambda_i|c^2\left(\eta T\right)^2 \\
 &\gee
  -2 (\rho/2)\eta^2 T^2 \eta T \\
 &\gee
  - \rho (\eta T)^3 \\
 &\gee
  - \tO(\eta^{3/2}) ~,
\end{align*}
here we used $\lambda_i |\alphaz^{(i)}| \leq (\rho/2)\eta^2 T^2$, $c\in[0,1]$, and $T=\tO(\eta^{-1/2})$.
\end{proof}

\subsection{Proof of Lemma~\ref{lem:Saddle_normGrows_General}} \label{Proof_lem:Saddle_normGrows_General}
\begin{proof}
Following is the key relation that enables us to prove the lemma:
\begin{align}\label{eq:GradEqNGD}
\nabla f(x_t) 
&\eq
 \nabla f(x_{t-1}) + \int_{0}^1 H(x_{t-1}+s(x_t -x_{t-1}))ds (x_t-x_{t-1}) \nonumber \\
&\eq
 \nabla f(x_{t-1}) + H_0(x_t-x_{t-1}) + \mu_t~,
\end{align}
where we denote  $\mu_t = \int_{0}^1 [H(x_{t-1}+s(x_t -x_{t-1}))-H_0]ds (x_t-x_{t-1})$.

Using the Saddle-NGD update rule inside Equation~\eqref{eq:GradEqNGD} we obtain:
\begin{align}\label{eq:GradEqNGD2}
\nabla f(x_t) 
\eq
\lr{I- \frac{\eta H_0}{\|\nabla f(x_{t-1}) \|}} \nabla f(x_{t-1}) +\mu_t- \theta H_0 n_t
~.
\end{align}
Due to the Lipschitzness of the Hessian, and since $\|x_t - x_{t-1} \|\leq \eta$, $\|x_t - x_0 \| \leq \eta t$
\footnote{In fact, due to the noisy updates, then with high probability we will have $\|x_t - x_0 \| \leq \eta t+ \sqrt{d}\theta\log(1/\xi)\frac{t}{N_0}$. Since $\theta = \tO(\eta),\; N_0 = \tO(\eta^{-1/2})$, choosing $\eta\leq \tO(1/d)$ we conclude that
$\sqrt{d}\theta\log(1/\xi)\frac{t}{N_0}=\tO(\sqrt{d}\eta^{3/2}t)\leq \eta t$. Thus having $\|x_t - x_0 \| \leq 2\eta t$.}
,  the last term is bounded by $\O(\eta^2 t)$:
\begin{align*}
 \|\mu_t\|  \lee \rho\eta^2 t~.
\end{align*}
Let us look at the gradient  component in the most negative direction $e_1$:
\begin{align}\label{eq:GradEqNGD2_Comp}
g_{t+1}^{(1)}  
\eq  
g_{t}^{(1)}  (1+\gamma \frac{\eta}{\|g_t\|})+\mu_t^\top e_1 + \gamma \theta n_t^{(1)} ~.
\end{align}
And recall from Algorithm~\ref{algorithm:SNGD} that $n_t$ is zero most rounds (non-zero once  every $N_0 = \tO(\eta^{-1/2})$ rounds); also recall that  for simplicity we assume that in $t=0$ the update is noisy. 

Next we will show that taking $| g_{1}^{(1)}|,N_0$ that fulfill the following two conditions imply that the  magnitude of the gradient rises beyond 
$2\beta\sqrt{\eta}$ within less than $N_0$ rounds:
\begin{align}
&\frac{\rho\eta^2 N_0}{| g_{1}^{(1)}|}  
\lee 
\frac{\gamma\sqrt{\eta}}{4\beta} \label{eq:cond1}~,\\
&N_0 
\gee
 \frac{4\beta}{\gamma\sqrt{\eta}} \log\left(\frac{2\beta\sqrt{\eta}}{| g_{1}^{(1)}|} \right) \label{eq:cond2}~.
\end{align}

Assume by contradiction that the gradient does not rise beyond $2\beta\sqrt{\eta}$ for $| g_{1}^{(1)}|,N_0$ that fulfill the above conditions. We will now show by induction that the following holds for any  $t\in \{1,\ldots, N_0-1\}$:
\begin{align}\label{eq:GradientInduction}
|g_{t}^{(1)}|
  \gee
 |g_{1}^{(1)}|(1+\frac{\gamma\sqrt{\eta}}{4\beta} )^{t-1}~.
\end{align}
The above clearly holds for $t=1$. Assume it holds for $t$, and we will now show that is holds for $t+1$. 
By Equation~\eqref{eq:GradEqNGD2_Comp}  we have:
\begin{align*}
|g_{t+1}^{(1)}|  
&\eq
 | g_{t}^{(1)}  (1+\gamma \frac{\eta}{\|g_t\|})+ \mu_t^\top e_1 | \\
&\gee
 | g_{t}^{(1)}  (1+\gamma \frac{\eta}{\|g_t\|})|- |\mu_t^\top e_1 |\\
&\gee
  | g_{t}^{(1)} |\left(1+\frac{\gamma\sqrt{\eta}}{2\beta} - \frac{\rho\eta^2 t}{| g_{t}^{(1)} |}\right) &\mbox{cont. assump. $\| g_t\|\leq 2\beta \sqrt{\eta}$}\\
&\gee
 | g_{t}^{(1)} |\left(1+\frac{\gamma\sqrt{\eta}}{4\beta}\right) &\mbox{Equation~\eqref{eq:cond1}}\\
&\gee
 |g_{1}^{(1)}|(1+\frac{\gamma\sqrt{\eta}}{4\beta} )^{t} &\mbox{Induct. hypothesis}
\end{align*}
here the second inequality uses $\|g_t\|\leq 2\beta\sqrt{\eta}$,  which is our contradiction assumption. The third inequality holds since $t\leq N_0$, and 
 by the induction hypothesis 
$|g_t^{(1)}|\geq|g_1^{(1)}|$, combining these with Equation~\eqref{eq:cond1} implies that
 $$\frac{\rho\eta^2 t}{| g_{t}^{(1)} |}
 \lee
 \frac{\rho\eta^2 N_0}{| g_{1}^{(1)}}
 \lee
 \frac{\gamma\sqrt{\eta}}{4\beta}~.$$
Thus the induction hypothesis holds for any  $t\in \{1,\ldots, N_0-1\}$.
The induction hypothesis implies that within less than $\frac{4\beta}{\gamma\sqrt{\eta}} \log\left(\frac{2\beta\sqrt{\eta}}{| g_{1}^{(1)}|} \right)$ rounds the magnitude of the gradient rises beyond $2\beta\sqrt{\eta}$, combining this with the condition of Equation~\eqref{eq:cond2} contradicts our assumption that the gradient does not rise beyond $2\beta\sqrt{\eta}$ within less that $N_0$. We therefore conclude that the gradient rises beyond $2\beta\sqrt{\eta}$ within less than $N_0$ rounds, for 
$| g_{1}^{(1)}|,N_0$ that fulfill conditions~\eqref{eq:cond1},\eqref{eq:cond2}.

It is rather technical to validate that conditions~\eqref{eq:cond1},\eqref{eq:cond2} are fulfilled by 
choosing $N_0 = \tilde{\Theta}(\eta^{-1/2})$ and $ | g_{1}^{(1)}| = \tOmega(\eta)$.
Particularly, the following choice fullfills these conditions:
\begin{align*}
|g_1^{(1)}| \eq	
 2\eta \left(\frac{4\beta\sqrt{\rho}}{\gamma}\right)^2 \log\left( \frac{\gamma^2}{8\beta\rho \sqrt{\eta}}\right), \text{ and   } 
N_0 
\eq
 \frac{4\beta}{\gamma\sqrt{\eta}} \log\left(\frac{2\beta\sqrt{\eta}}{| g_{1}^{(1)}|} \right)
\end{align*}

Note that we still need to ensure $|g_1^{(1)}|=\tOmega(\eta)$. Since we use symmetric gaussian noise, $ \theta n_0\sim  \theta \N(0,I_{d})$, then choosing $\theta = \tOmega(\eta )$, we guarantee the following to hold:
\begin{align*}
P\left(|g_{\tau}^{(1)}| 
\gee
 \tOmega(\eta)\right)\geq 1/3~.
\end{align*}
We can repeat the above process for $\log(1/\xi)$ epochs, ensuring that w.p.$\geq 1-\xi$, within $\tO(\eta^{-1/2})$ rounds we reach a point such that $\| g_t\|\geq 2\beta \sqrt{\eta}$.
\end{proof}

\subsection{Proof of Lemma~\ref{lem:ImproveLem_General}}
\label{Proof_lem:ImproveLem_General}
\begin{proof}
Recall that according to Lemma~\ref{lem:Saddle_normGrows_General} then the gradient goes beyond $2\beta \sqrt{\eta}$ within $t=\tO(\eta^{-1/2})$  steps.
In order to prove the lemma, we first relate $f$ to its quadratic approximation $\tf$ around $x_0$,  then 
we will show that $\tf(x_t) \leq \tf(x_0)-\Omega(\eta)$.

Given $x$, there always exists $x'\in[x_0,x]$ such that the following holds:
\begin{align*}
f(x) \eq
 f(x_0) + g_0^\top(x-x_0) +\frac{1}{2}(x-x_0)H_{x'}(x-x_0)~.
\end{align*}
where $H_{x'} = \nabla^2 f(x')$.
Using the above equation, and the Lipschitzness of the Hessian, we may bound the difference between the original function and its quadratic approximation around $x_0$ as follows:
\begin{align}\label{eq:preOriginal_QuadApprox}
|f(x) - \tf(x)| 
&\eq
\left| \frac{1}{2}(x-x_0)(H_{x'}-H_{x_0})(x-x_0)\right|\nonumber \\
&\lee
 \frac{\rho}{2}\|x-x_0\|^3~.
\end{align}
According to Lemma~\ref{lem:Saddle_normGrows_General}, 
w.p.$\geq 1-\xi$ we have $t\leq \tO(\eta^{-1/2})$. 
Combined with Equation~\eqref{eq:preOriginal_QuadApprox} we conclude that
\begin{align}\label{eq:Original_QuadApprox}
|f(x_t) - \tf(x)| 
&\lee
 \tO(\eta^{3/2})~.
\end{align}
We now turn to bound $\tf(x_t)-\tf(x_0)$. 
The bound requires the use of Corollary~\ref{cor:LambPositive_General}, which uses the following expression of 
$\tf(x)$ appearing also in Equation~\eqref{eq:QuadApprox2}:
$$\tf(x) \eq \frac{1}{2}x^\top H_0 x ~.$$
Thus, we use the above representation of $\tf$ in the remainder of the proof.
Using the eigen-decomposition according to  $H_0$ we may write: 
\begin{align*}
x_0 \eq \sum_i \alphaz^{(i)} e_i;
\quad 
x_t \eq \sum_i \alphat^{(i)} e_i~.
\end{align*}
Denoting $\tg_t:=\nabla \tf(x_t)=H_0 x_t$, we may also write:
\begin{align*}
\|\tg_0\|^2 \eq \sum_i \lambda_i^2 (\alphaz^{(i)})^2 ;
\quad 
\|\tg_t\|^2 \eq \sum_i \lambda_i^2 (\alphat^{(i)})^2~.
\end{align*}
Recall that since the quadratic approximation $\tf$ is taken around $x_0$, then $\tg_0 = g_0$. Also, utillizing 
Equation~\eqref{eq:GradError}, we conclude that the following holds for any $t\leq\tO(\eta^{-1/2})$:
\begin{align*}
\|\tg_t\|^2 
&\gee
 \| g_t\|^2 - \|g_t-\tg_t \|^2 \\
& \gee
  \| g_t\|^2- \frac{\rho^2}{4}(\eta t)^4 \\
& \ge
  \| g_t\|^2- \tO(\eta^2)~. 
\end{align*}
where we used $t\leq\tO(\eta^{-1/2})$.
Since we assume $\| g_t\|\geq 2\beta\sqrt{\eta}$,  the above implies that for a small enough $\eta$
we must have $\| \tg_t\|^2 \geq 3\beta^2\eta$.
Thus we have:
\begin{align}\label{eq:ValDescend1}
\tf(x_t)& -\tf(x_0) \nonumber\\
&= \sum_{\lambda_i\geq 0} \lambda_i\left( (\alphat^{(i)})^2-(\alphaz^{(i)})^2\right)\nonumber \\
&\quad+ \sum_{\lambda_i<0} \lambda_i\left( (\alphat^{(i)})^2- (\alphaz^{(i)})^2\right) \nonumber \\
&\lee
 d  \tO(\eta^{3/2})-\sum_{\lambda_i<0} |\lambda_i|\left( (\alphat^{(i)})^2- (\alphaz^{(i)})^2\right)
\end{align}
where we used Corollary~\ref{cor:LambPositive_General}.
We are now left to show that $\sum_{\lambda_i<0} |\lambda_i|\left( (\alpha_t^{(i)})^2- (\alpha_0^i)^2\right)= \Omega(\eta)$.
Using $\|\tg_0\|^2\leq \beta^2\eta$ and $\|\tg_t\|^2\geq 3\beta^2\eta$ we get:
\begin{align}\label{eq:ValDescend2}
&\beta \sum_{\lambda_i<0} |\lambda_i|\left( (\alpha_t^{(i)})^2- (\alpha_0^i)^2\right) \nonumber\\
&\gee
 \beta\sum_{\lambda_i<0} |\lambda_i| \left| (\alpha_t^{(i)})^2- (\alpha_0^i)^2\right|- 
d\tO(\eta^{3/2}) \nonumber\\
&\gee
 \sum_{\lambda_i<0} \lambda_i^2\left| (\alpha_t^{(i)})^2- (\alpha_0^i)^2\right|- 
d\tO(\eta^{3/2}) \nonumber\\
&\gee
 \|\tg_t \|^2 - \| \tg_0\|^2  \nonumber\\
 &\qquad- \sum_{\lambda_i\geq0} \lambda_i^2\left( (\alpha_t^{(i)})^2- (\alpha_0^i)^2\right)-d\tO(\eta^{3/2}) \nonumber\\
&\gee
 3\beta^2\eta - \beta^2 \eta -d  \tO(\eta^{3/2})  \nonumber \\
&\eq 2\beta^2\eta - d  \tO(\eta^{3/2})~,
\end{align}
where  the first inequality  uses Corollary~\ref{cor:LambNeg_General}, the second inequality
 uses the $\beta = \max_i |\lambda_i|$, and also Corollary~\ref{cor:LambPositive_General}.
Combining Equations~\eqref{eq:Original_QuadApprox}, \eqref{eq:ValDescend1}, \eqref{eq:ValDescend2}, and taking 
$\eta= \tO(1/d^2)$, the lemma follows.
\end{proof}


\end{document}